\documentclass[10pt,journal,compsoc]{IEEEtran}
\ifCLASSOPTIONcompsoc
  \usepackage[nocompress]{cite}
\else
  \usepackage{cite}
\fi

\ifCLASSINFOpdf
\else
\fi
\usepackage{epsfig}
\usepackage{helvet}
\usepackage{courier}
\usepackage{floatrow}
\usepackage{multirow}
\newfloatcommand{capbtabbox}{table}[][\FBwidth]
\usepackage{comment}
\usepackage[utf8]{inputenc}
\usepackage[T1]{fontenc}
\usepackage{amsfonts}
\usepackage{units}
\usepackage{microtype}
\usepackage[dvipsnames]{xcolor}
\usepackage{soul}
\usepackage{url}
\usepackage[small]{caption}
\usepackage{graphicx}
\usepackage{amsmath}
\usepackage{booktabs}
\usepackage{epstopdf}
\usepackage{amsthm}
\usepackage{amssymb}
\usepackage{subfigure}
\usepackage{wrapfig}

\newtheorem{myTheo}{Theorem}

\newtheorem{thm}{Theorem}[section]

\usepackage{bbding}
\usepackage{enumitem}
\usepackage{float}
\usepackage{multicol}
\usepackage[colorlinks,            
    linkcolor=red,
    anchorcolor=blue, 
    citecolor=green
]{hyperref}

\usepackage{colortbl}

\usepackage[ruled,vlined]{algorithm2e}
\usepackage{algpseudocode}

\usepackage[misc]{ifsym}

\usepackage[numbers]{natbib}

\begin{document}
%
\title{BiVM: Accurate Binarized Neural Network for Efficient Video Matting}
%
%
%
%

\author{Haotong~Qin, 
Xianglong~Liu,~\IEEEmembership{Senior~Member,~IEEE},
Xudong~Ma,
Lei~Ke,
Yulun~Zhang,
Jie~Luo,
Michele~Magno,~\IEEEmembership{Senior~Member,~IEEE}
\IEEEcompsocitemizethanks{
\IEEEcompsocthanksitem H. Qin and M. Magno are from ETH Zurich. X. Ma and J. Luo are from Beihang University. X. Liu (corresponding author, email: xlliu@buaa.edu.cn) is from the State Key Laboratory of Complex \& Critical Software Environment, Beihang University. L. Ke is from Carnegie Mellon University. Y. Zhang is from Shanghai Jiao Tong University.
}
}

\markboth{Journal of \LaTeX\ Class Files,~Vol.~14, No.~8, August~2015}%
{Shell \MakeLowercase{\textit{et al.}}: Bare Demo of IEEEtran.cls for Computer Society Journals}
%



\IEEEtitleabstractindextext{%
\begin{abstract}
Deep neural networks for real-time video matting suffer significant computational limitations on edge devices, hindering their adoption in widespread applications such as online conferences and short-form video production. Binarization emerges as one of the most common neural network compression approaches, substantially curbing computational and memory requirements through compact 1-bit parameters and efficient bitwise operations. However, the empirical observation reveals that accuracy and efficiency limitations exist in the binarized video matting network due to its degenerated encoder and redundant decoder. Following a theoretical analysis based on the information bottleneck principle, the limitations are mainly caused by the degradation of prediction-relevant information in the intermediate features and the redundant computation in prediction-irrelevant areas. We present \textbf{BiVM}, an accurate and resource-efficient \textbf{Bi}narized neural network for \textbf{V}ideo \textbf{M}atting, where architecture and optimization allow real-time video matting to proceed on edge hardware. First, we present a series of binarized computation structures with elastic shortcuts and evolvable topologies, enabling the constructed encoder backbone to extract high-quality representation from input videos for accurate prediction. Second, we sparse the intermediate feature of the binarized decoder by masking homogeneous parts, allowing the decoder to focus on representation with diverse details while alleviating the computation burden for efficient inference. Furthermore, we construct a localized binarization-aware mimicking framework with the information-guided strategy, prompting matting-related representation in full-precision counterparts to be accurately and fully utilized. Comprehensive experiments show that the proposed BiVM surpasses alternative binarized video matting networks, including state-of-the-art (SOTA) binarization methods, by a substantial margin. For example, BiVM surpasses 16.67 MAD compared to SOTA binarization on the VM dataset. Notably, our approach can even perform comparably to the full-precision counterpart in terms of visual quality. Moreover, our BiVM achieves significant savings of 14.3$\times$ and 21.6$\times$ in computation and storage costs, respectively. We also evaluate BiVM on ARM CPU hardware, underscoring its potential for deployment in resource-constrained scenarios. 
\end{abstract}

}

\maketitle

\IEEEdisplaynontitleabstractindextext

%
\IEEEpeerreviewmaketitle

\IEEEraisesectionheading{\section{Introduction}}
\label{sec:intro}
\IEEEPARstart{T}{he} remarkable progress of deep neural networks has revolutionized computer vision tasks, notably in video matting (VM)~\cite{lin2022robust,aksoy2017designing,bai2007geodesic,chen2013knn,chuang2001bayesian,feng2016cluster,lin2022robust,li2022vmformer,sun2021deep,zhang2021attention,lin2021real,sengupta2020background,wang2024matting,yang2025matanyone}. However, the demand for real-time video matting on edge devices in practical applications poses a serious challenge due to the high computational and storage requirements. To overcome this challenge, various lightweight video matting networks like Robust Video Matting (RVM)~\cite{lin2022robust} and BGMv2~\cite{lin2021real} have been developed. Despite achieving notable speedups and memory reductions, these methods rely on resource-intensive floating-point parameters and operations, leaving significant room for additional compression.

\begin{figure}[!t]
  \begin{center}
    \includegraphics[width=1.\textwidth]{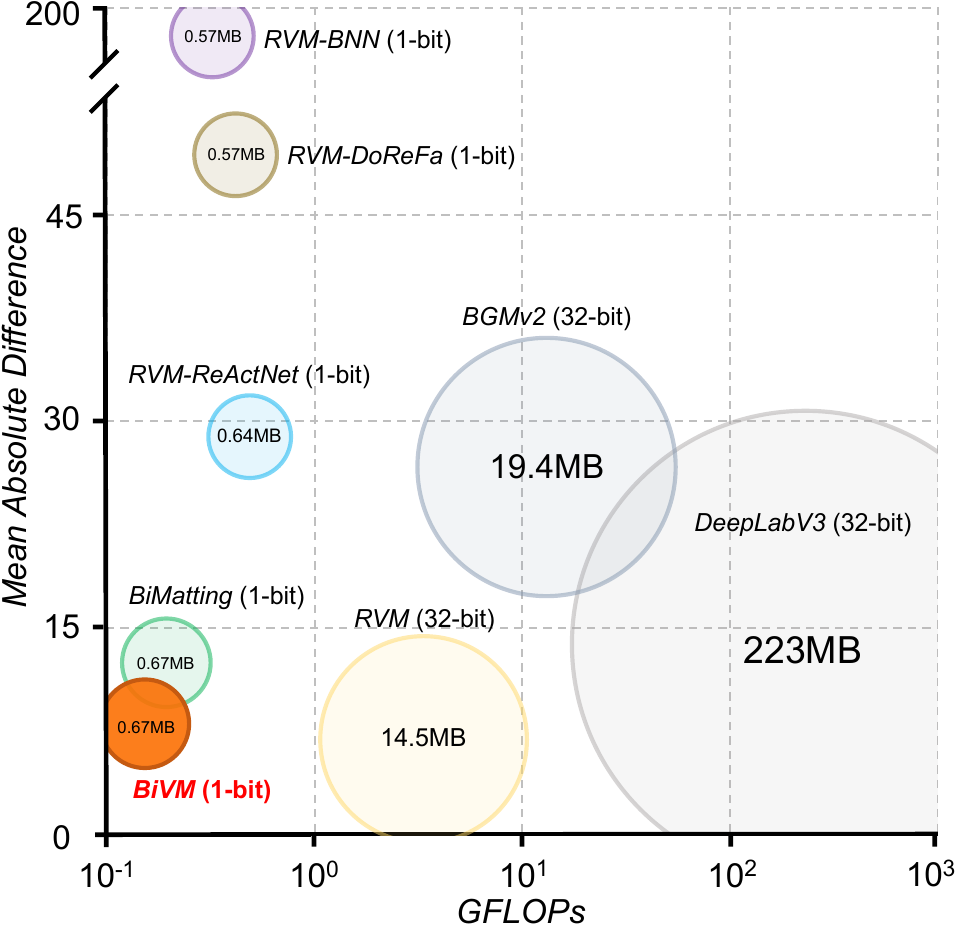}
  \end{center}
  \caption{BiVM achieves impressive resource savings, surpassing SOTA 1-bit and even 32-bit video matting networks in accuracy.}
  \label{fig:speed-size-acc}
\end{figure}

To address the pressing requirements of efficiency, researchers have explored network binarization~\cite{hubara2016binarized,dorefa,XNORNet,XNOR++,xu2021recu,BiReal,liu2020reactnet,shang2022lipschitz,shang2022network,bethge2020meliusnet} as one of the most promising neural network compression approaches. 
As the neural network with the most aggressive bit-width, binarized neural networks (BNNs) utilize compact 1-bit parameters to compress the model size. The binarized parameter can achieve 32$\times$ storage savings compared to the full-precision counterparts. Benefits from the binarized weights and activations, the expensive floating-point operations can be replaced by efficient bitwise ones, \textit{i.e.}, XNOR and POPCNT, during inference, achieving 58$\times$ speedup in 64-bit ARM CPU hardware for standard convolution~\cite{XNORNet}.
Adopting network binarization enables video matting applications to substantially save computational and storage costs and move towards real-time inference on resource-constrained devices.

Despite the extensive exploration of generic binarization methods, the binarized video matting networks still suffer substantial limitations in terms of accuracy and efficiency:
\textbf{For the accuracy limitation}, the direct binarization for the lightweight encoders of video matting networks causes catastrophic accuracy degradation. State-of-the-art (SOTA) video matting networks incorporate encoders with lightweight backbones to extract intermediate representations. For example, the MobileNet families are commonly employed as lightweight encoders in various video matting networks, contributing extracted features with different scales for feeding the following video matting decoders~\cite{lin2022robust,sengupta2020background}. However, directly binarizing these architecture variants leads to almost complete failure in extracting feature functions, which introduces catastrophic accuracy decline or even collapse in the binarized video matting network. The empirical accuracy analysis in Figure~\ref{fig:observation}(a) quantitatively indicates that the binarization of the lightweight backbone of the encoder is the primary cause of the accuracy degradation, implying that the feature extraction functionality is severely impaired after binarization.
\textbf{For the efficiency limitation}, the significant computational redundancy of the decoder is hard to eliminate directly through binarization. In the existing binarized video matting baseline, the decoder's computation is processed on feature maps with multiple scales extracted by the encoder. However, since the application of intensive computation in large portions of certain foreground or background regions (spatial regions outside unknown parts of the trimap) of the high-resolution feature maps, a large amount of redundant computation counts exists in the decoder~\cite{lin2022robust}. Though the computational resources are significantly saved with the reduced bit-width, the efficiency benefits of binarization are weakened due to the constant redundant computation counts. Quantitative efficiency analysis in Figure~\ref{fig:observation}(b) reveals that the redundant decoder hinders efficiency improvements in extremely compressed bit-width. 

In this paper, our empirical studies present the aforementioned accuracy and efficiency challenges for binarizing video matting networks.
Based on the information bottleneck principle, we attribute the primary challenges to 1) the binarized encoder cannot effectively compress prediction-relevant information into the intermediate representation and 2) the significantly redundant computation of the binarized decoder exists in background areas with massive prediction-irrelevant information.
Our observations and analysis lead us to present \textbf{BiVM}, a novel \textbf{Bi}narized neural network designed to enhance the accuracy and efficiency of \textbf{V}ideo \textbf{M}atting.
First, we delve into the limitations of existing binarized encoders in feature extraction and propose the \textit{Evolvable Binarized Block} (EBB), a binarized structure paradigm with evolvable elastic shortcuts. With the improved topologies, EBB allows the constructed lightweight binarized backbone to effectively extract sufficient prediction-relevant representation from input videos to achieve accurate matting.
Second, we introduce the \textit{Sparse Heterogeneous Binarization} (SHB) to diminish the redundant cost of the binarized encoder by skipping the feature computation for background on different scales. SHB reduces redundant prediction-irrelevant computation through spatial sparsification~\cite{transfiner}, resulting in a notable reduction in computation while focusing on the critical information in the representation.
Moreover, we construct an \textit{Localized Binarization-aware Mimicking} (LBM) framework for optimizing the binarized video matting network. LBM is assisted by the full-precision features and guided by the information-maximized mask, prompting the critical knowledge to be accurately learned for the binarized video matting network.

Our BiVM is a novel binarized neural network for efficient video matting that stands out from traditional 1-bit matting models utilizing established binarization algorithms in previous methods~\cite{liu2020reactnet,hubara2016binarized,xu2021recu,dorefa}. BiVM demonstrates significant superiority even compared to some floating-point video matting networks~\cite{sengupta2020background,chen2017rethinking} in terms of accuracy while maintaining impressive efficiency. 
Extensive experiments conducted across fundamental tasks on VideoMatte240K (VM)~\cite{lin2021real}, Distinctions-646 (D646)~\cite{qiao2020attention}, and Adobe Image Matting (AIM)~\cite{xu2017deep} datasets establish that specific tasks do not constrain the advantages of BiVM. 
For example, compared to the video matting network binarized by ReActNet, a SOTA generic binarization algorithm, and BiMatting, BiVM achieves up to 16.67 and 1.00 Alpha-MAD accuracy improvement on the VM dataset, respectively, and even achieves 13.37 gains against the 32-bit BGMv2 network. 
Furthermore, our proposed structures showcase remarkable efficiency, enabling BiVM to achieve noteworthy 14.3$\times$ FLOPs and 21.6$\times$ storage savings compared to its full-precision counterpart and also achieve 13.2$\times$ latency reduction compared to the full-precision counterpart on mobile devices with 64-bit ARM CPUs. This opens up promising possibilities for video matting in edge scenarios (refer to Figure~\ref{fig:speed-size-acc}).

We summarize our main contributions as follows:
\begin{itemize}
\item We empirically reveal the accuracy and efficiency limitations of binarized video matting networks, showcasing the degradation of the binarized encoder backbone and the redundant computation in the binarized decoder.
\item We provide a theoretical perspective for the limitations of the binarized video matting baseline based on the information bottleneck principle, pointing out that the degradation of prediction-relevant information in the intermediate features and the redundant computation in prediction-irrelevant areas mainly hinder the accurate and efficient video matting, respectively.
\item We propose a novel binarized video matting framework, namely BiVM, to push the binarized video matting network to be accurate and efficient through comprehensive improvements of typology and optimization.
\item For the encoder structure, we propose a flexible binarized topology paradigm, namely Evolvable Binarized Block (EBB), to enhance the feature extraction capability of the encoder by evolvable and elastic shortcuts and extract features to prediction-relevant.
\item For the decoder structure, we propose an efficient sparse binarization technique to exclude the prediction-irrelevant information, namely Sparse Heterogeneous Binarization (SHB), allowing the binarized decoder to utilize the representations with low computation.
\item For the training strategy, we design an information-guided optimization scheme, namely Localized Binarization-aware Mimicking (LBM), learning critical representations of the external full-precision counterpart to improve the binarized video matting network.
\item {We present detailed quantitative and visual results to verify the effectiveness of techniques in BiVM regarding accuracy and efficiency.} Comprehensive experiments show that BiVM consistently outperforms video matting networks binarized by SOTA methods in accuracy and presents significant efficiency advantages on both theoretical and edge hardware.
\end{itemize}

Note that our paper expands on the original conference version~\cite{qin2024bimatting}. 
\textbf{1)} This paper further offers a theoretical explanation for the empirical observations based on the information bottleneck principle, pointing out that the lack of prediction-relevant information in intermediate features and the redundant computation of prediction-irrelevant parts are the main reasons for the challenges in accuracy and efficiency faced by the existing binarized video matting network. 
\textbf{2)} This work proposes a novel binarized neural network for accurate and efficient video matting, BiVM, including novel encoder and decoder structures and training strategies. The previous solution proposed in the original conference paper can be seen as a specified baseline method. Specifically, (a) in terms of structural design, we propose a flexible and evolvable information-elastic topology for the encoder with improved feature extraction capabilities compared to the original method. (b) We propose an efficient sparse binarization technique focused on high task-relevant information for the decoder that achieves higher accuracy with less computation than the original method. (c) We also designed a binarization-aware distillation scheme for training guided by the information-maximized mask, which improves the information representation of binarized video matting by focusing critical representations from an external full-precision counterpart. 
\textbf{3)} We further evaluate and analyze the binarized video matting network in detail, including quantitative and visualization results on video and image matting datasets and the visualization analysis of the information plane during the training process, demonstrating significant accuracy and efficiency improvements brought by BiVM. We also conduct the real-world ARM CPU hardware evaluation of the proposed BiVM, demonstrating its efficiency potential in practical deployment scenarios.

\section{Related Work}

\subsection{Video Matting}
Recent advancements in deep learning have significantly propelled the fields of matting forward. For image matting, neural network-based techniques estimate alpha mattes from image feature maps with auxiliary trimap supervisions in an end-to-end manne~\cite{chuang2001bayesian,bai2007geodesic,gastal2010shared,he2011global,chen2013knn,karacan2015image,feng2016cluster,aksoy2017designing,levin2007closed}. 
\textcolor{black}{Some researchers explore trimap-free solutions utilizing segmentation networks~\cite{chen2018semantic,ke2022modnet} or coarse annotations~\cite{liu2020boosting,yu2021mask}.}
Video matting is a recent development that benefits from integrating temporal information to enhance matting quality~\cite{wang2024matting,zou2019unsupervised,wei2023deep,fan2011scribble}.
Trimap-based methods~\cite{sun2021deep,zhang2021attention} focus on spatial-temporal feature aggregation for temporal feature fusion and alignment. 
Trimap-free methods~\cite{sengupta2020background,ke2020green} leverage an auxiliary background image for the initial frame to provide crucial cues for alpha matte prediction. 
BGMv2~\cite{lin2021real} extends this approach by offering high-resolution real-time video matting solutions. 
\textcolor{black}{MODNet~\cite{ke2022modnet} proposes a light-weight and auxiliary-free matting model with neighbor frame smoothing.}
RVM~\cite{lin2022robust} enhances matting quality robustness on real-world data through training on segmentation data. 
\textcolor{black}{MatAnyone~\cite{yang2025matanyone} further advances auxiliary-free video matting with a robust target-assigned framework that employs consistent memory propagation and a novel training strategy on a large diverse dataset, achieving strong performance and detail preservation in real-world scenarios.}
From an architecture perspective, VideoMatt~\cite{li2023videomatt} explores various temporal modeling methods based on CNNs,
and some methods adopt vision transformer architectures~\cite{heo2022vita,wu2022seqformer,li2022vmformer} in the video matting task.

Though many models have been proposed and built, running video matting tasks in real-time under limited resource environments still faces challenges. 
\textcolor{black}{Among existing methods, some lightweight video matting algorithms, such as RVM~\cite{lin2022robust}, MODNet~\cite{ke2022modnet}, and BGMv2~\cite{lin2021real}, have been proposed.}
Although significant speedups and memory reductions are achieved, these video matting methods still rely on resource-intensive floating-point operations, leaving room for further bit-width compression. \cite{qin2024bimatting} proposes the first 1-bit binarized video matting methods, leading a promising way to efficient video matting. However, a non-negligible accuracy degradation is caused by the binarization of the video matting network compared to the full-precision counterpart.

\subsection{Network Quantization}
By considering the bit-width perspective, network quantization emerges as one of the most popular network compression approaches. Two main quantization paradigms exist: Post-Training Quantization (PTQ) and Quantization-Aware Training (QAT)~\cite{Qin_2020_pr,gholami2022survey,qin2023diverse}.

PTQ methods can efficiently reduce the computational overhead of quantizing neural network models without fine-tuning~\cite{dong2019hawq}. For instance, bias correction methods address inherent biases in weight values post-quantization, and techniques like ACIQ~\cite{banner2018aciq} and OMSE~\cite{choukroun2019low} optimize clipping ranges and L2 distance to reduce errors. Outlier channel splitting (OCS)~\cite{zhao2019improving} and adaptive rounding methods like AdaRound~\cite{nagel2020up} and AdaQuant~\cite{hubara2021accurate} enhance PTQ performance by addressing specific quantization challenges. 
PTQ is particularly advantageous for its speed and low overhead, but restricts the accuracy potential of quantized models.
Fortunately, QAT methods allow us to utilize the entire training pipeline to achieve aggressive low-bit quantization~\cite{courbariaux2015binaryconnect,hubara2016binarized,lin2015neural,huang2404empirical,qin2023distribution,zhang2021diversifying,qin2024accurate}. 
In QAT, the forward and backward passes are performed in floating-point precision, while the model parameters are quantized after each gradient update~\cite{XNORNet,tailor2020degree}. 
A key challenge in QAT is dealing with the non-differentiable quantization operator, often addressed by approximating the gradient of the quantization operation~\cite{bengio2013estimating,bai2018proxquant,yin2019understanding,rosenblatt1961principles,fan2020training, zhuang2018towards}. While STE is widely used, other methods and regularization have also been explored~\cite{agustsson2020universally,cai2017deep,BiReal}. Additionally, some QAT methods learn quantization parameters, such as clipping ranges and scaling factors, during training. 
QAT pushes the quantization to lower bit-width and demonstrates promising performance, even to extreme 1-bit binarization~\cite{XNORNet,XNOR++,martinez2020training}.

\begin{figure*}[t]
  \centering
  \includegraphics[width=1.\textwidth]{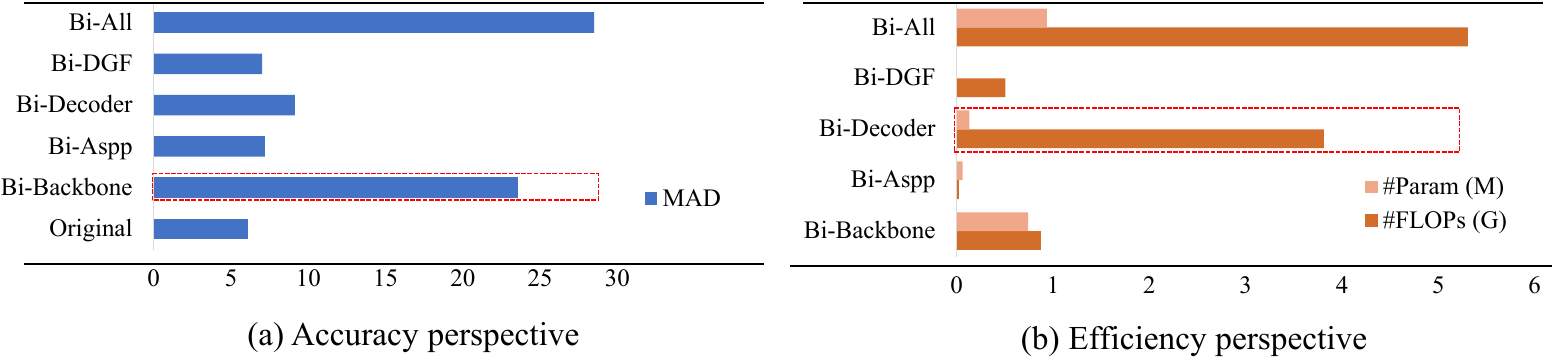}
  \caption{Limitations of the binarized video matting baseline in accuracy and efficiency on the VM dataset. We present (a) the MAD metrics for binarizing each component and (b) the FLOPs and storage requirements for each component in the binarized baseline.}
  \label{fig:observation}
\end{figure*}

\subsection{Binarized Neural Network}
Existing studies have demonstrated that 1-bit quantization, namely binarization, is a highly efficient form of network quantization~\cite{hubara2016binarized,wang2020bidet,wang2021learning,wang2019learning,Qin_2020_pr,wang2023binary,wang2021gradient,bulat2018hierarchical,lin2022siman,qin2022bibert,huang2024billm,}. This technique compresses networks significantly, enabling extreme storage and computational efficiency through the use of 1-bit binarized parameters and bitwise XNOR-POPCNT operators.

As binarization methods evolve, they can compress and accelerate deep neural networks across different neural architectures and learning tasks~\cite{qin2020forward,shang2022lipschitz,cai2024binarized}. Several strategies were proposed to enhance the universal performance of binarized networks. These include minimizing quantization error by reducing information loss during the conversion of 32-bit values to 1-bit~\cite{XNORNet}; improving loss functions to narrow the accuracy discrepancy with full-precision networks~\cite{Regularize-act-distribution} or employing knowledge distillation from full-precision to binarized networks~\cite{du2024bitdistiller}; and designing gradient approximations for backward propagation to match the gradient of the binarization function~\cite{RealtoBin}.
Tailoring specific binarization algorithms or structures for each network and task can enhance performance. This is evident in different network types such as CNNs~\cite{BiReal,liu2020reactnet}, transformer-based~\cite{bai2021binarybert,liu2022bit}, and MLP-based~\cite{xu2021poem} networks. However, for video matting networks, especially those that adopt the lightweight backbone (\textit{e.g.}, MobileNetV3), generic binarization approaches often fall short due to stringent accuracy requirements, and specialized binarized networks still experience significant accuracy gaps.

\section{BiVM: Binarized Video Matting Network}

\subsection{Preliminaries}

We establish a binarization baseline for video matting constructed from an advanced generic binarization approach and a lightweight video matting architecture.

\subsubsection{Network Binarization Framework}
\label{lab:NetworkBinarizationFramework}
The application of binarization mainly involves utilizing the sign function to compress weights and activations to 1-bit values for the computation units within the binarized network~\cite{hubara2016binarized}. The forward process can be expressed as
\begin{equation}
\label{eq:quantized}
\operatorname{sign}(\boldsymbol{x})=
\begin{cases}
+1,& \mathrm{if} \ \boldsymbol{x} \ge 0\\
-1,& \mathrm{otherwise}
\end{cases},
\end{equation}
and the backward process is
\begin{equation}
\label{eq:quantized_back}
\frac{\partial \mathcal{L}}{\partial \boldsymbol{x}}=
\begin{cases}
\frac{\partial \mathcal{L}}{\partial \operatorname{sign}(\boldsymbol{x})},& \mathrm{if} \ x \in \left(-1, 1\right)\\
0,& \mathrm{otherwise}
\end{cases},
\end{equation}
where $\mathcal{L}$ denotes the loss function.
To construct an advanced binarized baseline using existing methods, besides the standard sign function for binarization, we employ floating-point scaling factors for the weight $\boldsymbol{w}$~\cite{XNORNet}, along with adaptable thresholds for the activation $\boldsymbol{a}$~\cite{liu2020reactnet}:
\begin{equation}
\boldsymbol{b_w}=\operatorname{sign}(\boldsymbol{w}),\qquad
\boldsymbol{b_a}=\operatorname{sign}(\boldsymbol{a}-\tau),
\end{equation}
where $\boldsymbol{b_w}$ and $\boldsymbol{b_a}$ denote binarized weight and activation representations, respectively. The bitwise computation is
\begin{equation}
{o}=s\boldsymbol{b_w}\otimes\boldsymbol{b_a},
\end{equation}
where $s$ is the scalar which is obtained by $\operatorname{mean}(|\boldsymbol{w}|)$, while the initial value for layer-wise threshold $\tau$ is set to 0, and $\otimes$ denotes the bitwise convolution operator.
In the backward propagation, Straight-Through Estimator (STE)~\cite{bengio2013estimating} is consistently used across all settings, which approximates the gradient of the sign function as a clip function (as the backward process in Eq.~\eqref{eq:quantized_back}). 
Following the common practice, binarization is excluded in the computational layers that interact with original inputs or generate final outputs~\cite{BiReal,qin2023bibench}.

\subsubsection{Video Matting Architecture}
\label{lab:VideoMattingArchitecture}
The goal of the matting network is to decompose a frame $I$ into a foreground $F$ and a background $B$. This is achieved by utilizing an $\alpha$ coefficient to express the linear combination of the two of them~\cite{li2022vmformer,lin2022robust}, which can be expressed as
\begin{equation}
I=\alpha F+(1-\alpha) B.
\end{equation}
In contrast to various image matting approaches~\cite{aksoy2017designing,bai2007geodesic,chen2013knn,chuang2001bayesian,feng2016cluster,grady2005random,chen2018semantic,liu2020boosting,yu2021mask}, matting techniques specifically designed for video are anticipated to exhibit enhanced efficacy by leveraging spatial-temporal information inherent~\cite{lin2022robust,li2022vmformer,sun2021deep,zhang2021attention,lin2021real,sengupta2020background,ke2020green,gu2022factormatte}. Among them, the recent state-of-the-art Robust Video Matting (RVM)~\cite{lin2022robust} attains superior accuracy in video matting tasks, showcasing remarkable efficiency.

Despite being one of the most lightweight video matting models, RVM architecture still holds substantial potential for bit-width compression, given the considerable expenses associated with floating-point parameters and computations. Specifically, the RVM is primarily constructed by a streamlined encoder based on MobileNetV3 (comprising both the backbone and Atrous Spatial Pyramid Pooling (ASPP) module~\cite{chen2017deeplab}). As for the backbone, we can express the blocks in binarized MobileNetV3 as follows:
\begin{equation}
\phi_\textrm{mbv3}(\boldsymbol{x})=\operatorname{bconv}_1 (\operatorname{bconv}_n^\textrm{G}(\operatorname{bconv}_1 (\boldsymbol x)))+[c_{\boldsymbol{x}}=c_{\boldsymbol{o}}]\boldsymbol x,
\end{equation}
where $\operatorname{bconv}$ denotes binarized convolution, the subscript $n$ (or a specific number) denotes the kernel size, the superscript $\operatorname{G}$ denotes the groupwise convolution, and $\boldsymbol x$ denotes the input. The kernel size is denoted by $n$ where $n$ takes values from the set $\{5, 3\}$. Additionally, $[\cdot]$ represents the \textit{Iverson bracket}~\cite{iverson1962programming}, which results in 1 when the condition within the parentheses is true and 0 otherwise. The activation layers and batch normalization following convolutions are excluded. Except for the encoder, the binarized video matting baseline includes a recurrent decoder and a deep guided filter (DGF) module~\cite{wu2018fast}, where the convolutions in these parts are also binarized correspondingly for efficient computation.

\subsection{Challenges of Video Matting Network Binarization}
While we construct a binarized video matting network baseline and its training pipeline using standard techniques, this approach still suffers significant limitations in accuracy and efficiency, which impede its practical application.

\subsubsection{Limitations of Accuracy and Efficiency}
\label{subsec:BottlenecksofBinarized}
We aim to attain accurate and efficient video matting through binarization, focusing on practicality and resource utilization. Nevertheless, the encoder and decoder of binarized baseline suffer from accuracy and efficiency limitations.

\textbf{Accuracy Limitation:} In terms of accuracy, Figure~\ref{fig:observation}(a) illustrates the accuracy drops on the VideoMatting (VM) dataset by binarizing different parts in the RVM model. Notably, binarizing the lightweight MobileNetV3 backbone in the encoder causes the most significant accuracy drop, nearly equivalent to direct binarization of the entire network (Bi-All 28.49 \textit{vs.} Bi-Backbone 23.56 for MAD metric). In contrast, binarizing ASPP, decoder, and DGF parts has a lesser impact on accuracy (less than a 3.05 MAD increase). Hence, prioritizing improvement in the encoder's backbone for better binarization compatibility becomes essential to address accuracy drops.

\textbf{Efficiency Limitation:} Figure~\ref{fig:observation}(b) showcases the computation and storage usage for each part of the binarized video matting network to highlight individual efficiency consumption. Efficiency analysis reveals that the decoder utilizes a significant amount of computational resources, even constituting 71.8\% (0.38G) FLOPs with only 12.1\% parameters with binarization. On the contrary, the binarized encoder has a higher parameter count (81.8\%), but its FLOP consumption is only 16.6\%. Minimal consumption is observed for other parts. These observations indicate that the decoder's computational redundancy significantly impacts overall model acceleration performance, leaving room for computational optimization even after direct binarization.

Inspired by the aforementioned empirical observations, two significant challenges emerge in the current baseline: 1) the current lightweight backbone in the encoder is unsuitable for direct binarization, hindering the generation of practical representations, and 2) the decoder exhibits considerable computational redundancy after binarization. Consequently, we propose a novel binarized video matting network that includes a binarization-friendly encoder, a computationally efficient decoder, and an effective optimization scheme for accurate and efficient video matting.

\begin{figure}[!t]
  \begin{center}
    \includegraphics[width=1.\textwidth]{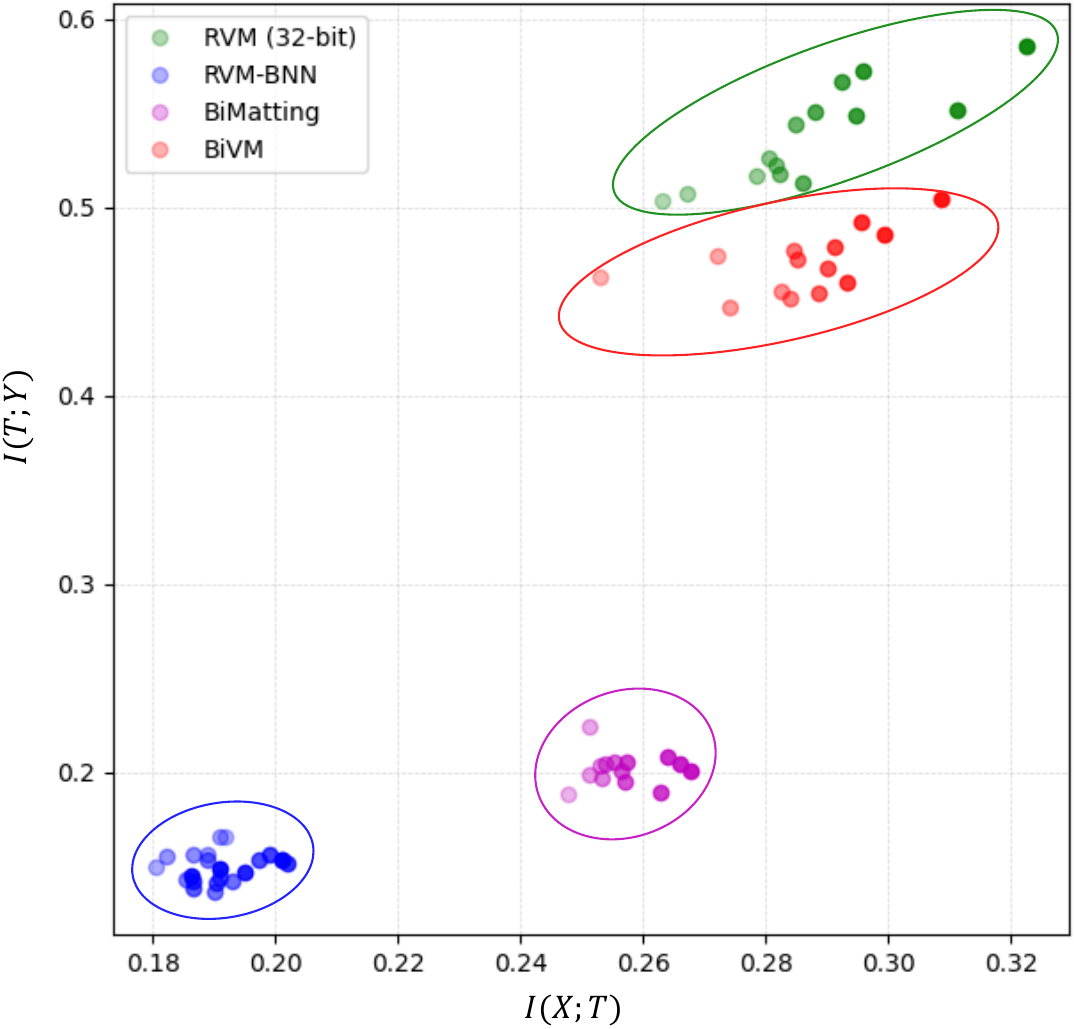}
  \end{center}
  \caption{Information planes of RVM, RVM-BNN, BiMatting, and BiVM on test sets. Darker dots represent later epochs.}
  \label{fig:infor_plane}
\end{figure}

\subsubsection{Analysis from the Information Perspective}
\label{subsec:AnalysisfromtheInformationPerspective}
We follow the Information Bottleneck (IB) principle to explain the performance degradation of video matting networks caused by binarization and to inspire designs that improve accuracy and efficiency. A video matting network $M$ can be viewed as a Markov chain $Y \rightarrow X \rightarrow T$, where $X$ and $Y$ are the video input and the ground truth of the matting, respectively. The intermediate representation $T$ is defined by the encoder $\operatorname{P}_{\theta_e}(T|X)$ and the decoder $\operatorname{P}_{\theta_d}(Y|T)$, where $\theta_e$ and $\theta_d$ represent the parameters of the encoder and decoder, respectively.
The goal of training the video matting network is to learn the optimal representation $T$ of $X$ with respect to $Y$, which can be expressed as:
\begin{equation}
T = \underset{S(X): I(S(X); Y) = I(X; Y)}{\arg \min} I(S(X); X).
\end{equation}
In this context, a sufficient statistic of $X$, denoted by $S(X)$, refers to a partition of $X$ that retains all the information that $X$ contains about $Y$, meaning that $I(S(X); Y) = I(X; Y)$. The Information Bottleneck (IB) principle relaxes this objective to find the optimal balance between compressing $X$ and accurately predicting $Y$. The goal is to transmit all the relevant information from $X$ about $Y$ through a bottleneck created by the compressed representation $T$~\cite{slonim2002information, shwartz2022information}. Thus, the task of determining the compressed representation $T$ of $X$ translates into minimizing a Lagrangian objective:
\begin{equation}
\label{eq:ib}
\min_{\theta_e, \theta_d} I(X; T) - \beta I(T; Y),
\end{equation}
where $\beta$ is a Lagrange multiplier. For a binarized video matting network, the intermediate representation $\hat{T}$ is defined by the encoder $\operatorname{P}_{\hat{\theta_e}}(T|X)$ and the decoder $\operatorname{P}_{\hat{\theta_d}}(Y|T)$, where $\hat{\theta_e}$ and $\hat{\theta_d}$ represent the parameters of the binarized encoder and decoder, respectively.

The binarized video matting network suffers a significantly degenerated representation capability compared to its full-precision counterpart, hindering the prediction of $Y$ given $X$. Figure~\ref{fig:infor_plane} shows the changes in $I(X; T)$ and $I(T; Y)$ during the training process of the full-precision and the binarized video matting networks (RVM-BNN, BiMatting, and our BiVM), which follows~\cite{shwartz2017opening,wang2020bidet} and covers epochs of four matting training stages. In Figure~\ref{fig:infor_plane}, $I(X; T)$ and $I(T; Y)$ of RVM-BNN are correspondingly much lower than in the full-precision network. This suggests that in directly binarized networks, the information in $X$ is extremely compressed during encoding, and the extracted features $T$ contain less information related to the prediction $Y$ and/or are not correctly decoded.

\begin{myTheo}
\label{theo:th2}
Let $X$ be a normally distributed random variable. 
Consider the two functions
$\hat{T} = f(\operatorname{sign}(X))$ and $T = f(X)$,
where $\operatorname{sign}(\cdot)$ denotes the sign function and $f(x) = ax + b\ (a \neq 0)$, then we have $I(X; \hat{T}) \ll I(X; T)$.
\end{myTheo}
Theorem~\ref{theo:th2} demonstrates that extreme binarization functions cause the severe restriction of the information flow in the network. This implies that the significant decrease of $I(X; T)$ and $I(T; Y)$ mainly arises from the binarization functions in the network. Although this aligns with the optimization objective of the encoder that minimizes $I(X; T)$ as described in Eq.~\eqref{eq:ib}, it also leads to the significantly reduced upper bound of $I(T; Y)$ for the decoder. This means the extracted representation $T$ does not include enough relevant information about $Y$ for accurate inference, and the decoding computations of prediction-irrelevant features can be considered significantly redundant.

Overall, our analysis confirms our empirical results: the application of the binarization function in the encoder leads to a restriction in the information flow and a low correlation between the extracted features and the matting predictions; the decoder has significant computational redundancy to decode the low-correlation parts of the extracted features. We propose an accurate binarized network, BiVM, to solve the challenges above for efficient video matting.

\begin{figure*}[t]
  \centering
  \includegraphics[width=1.\textwidth]{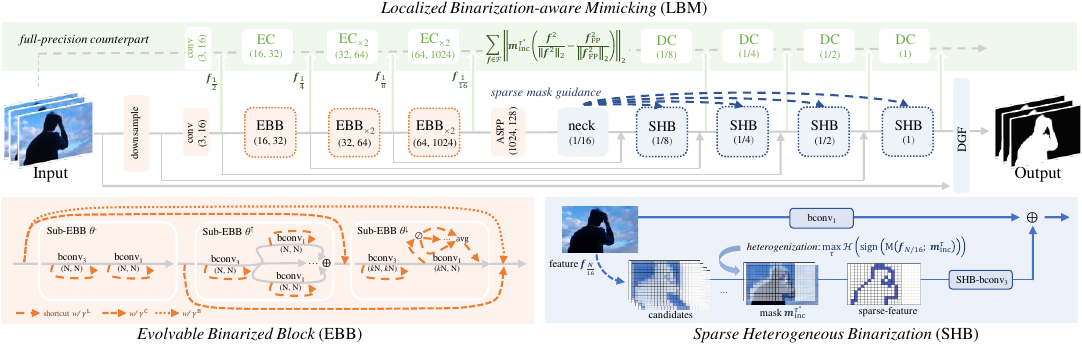}
  \caption{The overview of BiVM. We apply the Evolvable Binarized Block (EBB) for the encoder, the Sparse Heterogeneous Binarization (SHB) for the decoder, and the Localized Binarization-aware Mimicking (LBM) for the training process. Arrows $\rightarrow$ indicate feature flow, $\oplus$, $\oslash$, and $\operatorname{avg}$ indicate concatenation, splitting, and averaging features, respectively.}
  \label{fig:overview}
\end{figure*}

\subsection{BiVM: Evolvable Binarized Block}

\subsubsection{Binarization-Evoked Encoder Degradation}
\label{sec:Binarization-evokedEncoderDegradation}
The binarization baseline for video matting, as detailed in Section~\ref{lab:VideoMattingArchitecture}, employs a compact MobileNetV3 architecture as the encoder backbone with directly binarized convolutional and linear layers. But as demonstrated in Figure~\ref{fig:observation}, this direct binarization approach results in a degradation of over 4$\times$ in the MAD metric for the whole network. 

\begin{myTheo}
\label{theo:th3}
Let $X$ be an input variable and $\{f_k : k \in \mathbb{Z} \cap [1, N]\}$ be a set of functions. Define the representation $T_n$ as:
$T_n = \prod_{k=1}^n f_k(X)$
Then, for any $i, j \in \mathbb{Z} \cap [1, N]$ with $i < j$, we have:
$I(X; T_i) \geqslant I(X; T_j)$.
\end{myTheo}
Based on the analysis in Section~\ref{subsec:AnalysisfromtheInformationPerspective}, one of the main reasons for accuracy degradation is the extreme compression of $I(X; T)$ caused by the binarization function, which disrupts the information bottleneck trade-off among $X$, $T$, and $Y$. Thus, it is necessary to ensure that the obtained $T$ retains more prediction-relevant information to improve $I(T; Y)$. Introducing direct shortcuts in the binarized backbone network is an intuitive approach. Theorem~\ref{theo:th3} shows that the representation feedforwarded by the more direct transformation function can preserve more amount of the original information from input $X$ for accurate prediction by preventing $X$ from being subject to the extreme discretization effect of the binarized branch of the network. One of the most direct transformation functions is the shortcut for consistent input and output dimensions or the fixed pooling mappings for inconsistent dimensions.

The unsatisfactory performance of existing binarized video matting networks indicates that solely relying on a single short connection for each block is insufficient to improve the information flow and prevent accuracy collapse in binarized baselines. Some existing works attempt to construct per-convolution shortcuts~\cite{phan2020binarizing,liu2020reactnet}, but they struggle to adapt to channel variations in convolutions without significantly increasing computational burden. Additionally, their fixed feature mappings and strict convolutional co-location make it difficult to flexibly handle optimization challenges caused by extreme discretization, particularly the error accumulation across multiple binarized computational units.

Another reason for degradation is that in the binarized MobileNetV3 blocks, all convolutions are grouped or pointwise, with fewer parameters than regular convolutions, making them sensitive to binarization and hard to recover from collapse caused by mutual interference~\cite{phan2020binarizing,liu2020reactnet}. However, when considering potential alternatives, we find that existing binarized architectures face significant challenges to achieving the efficiency demonstrated by binarized MobileNetV3, making their direct application to a binarized matting network infeasible. For instance, some advanced binarized video matting architectures with a similar scale of parameters exhibit at least 3-8$\times$ higher FLOPs compared to binarized MobileNetV3~\cite{liu2020reactnet,phan2020binarizing,bethge2020meliusnet}.

Therefore, there is a pressing need for a novel binarized backbone that can effectively facilitate feature extraction, ensuring high-quality features for the decoder while maintaining an ultra-lightweight network architecture.

\subsubsection{Evolvable Binarized Block for Accurate Encoder}
To mitigate the performance degradation of the encoder caused by binarization while maintaining a lightweight structure, we propose an Evolvable Binarization Block (EBB) to construct a binarized video matting encoder (as in Figure~\ref{fig:overview}). The EBB consists of a computation structure robust to binarization and introduces multi-level evolvable shortcuts to improve the information flow in the binarized backbone.

Based on the analysis presented in Section~\ref{sec:Binarization-evokedEncoderDegradation}, the key to ensuring the accuracy of the binarized encoder is to improve its information flow to retain input-relevant information. Therefore, we define the following structural paradigm: a binarization-robust block structure composed of regular and pointwise convolutions, incorporating evolvable shortcuts at the layer-wise, cross-layer-wise, and block-wise levels. In terms of computational units, using computationally intensive regular convolutions instead of grouped convolutions can prevent the performance collapse caused by the mutual exclusion of grouped and pointwise convolutions after binarization. Regarding the topology, this paradigm establishes feature direct connections for binarized convolutions at various granular levels, achieving a reasonable balance between retaining prediction-relevant information and accurate information compression in the proposed evolvable strategy. Following this paradigm, the EBB reconstructs representations within binarized blocks, enhancing the backbone's ability to extract efficient features in a flexible dimensional space.

Specifically, the proposed EBB is mainly constructed by binarized convolutional layers without grouping and evolvable feature shortcuts. 
The sub-blocks of EBB can be expressed as follows, which can be seen as the basic units to construct the EBB backbone:
\begin{equation}
\begin{aligned}
\label{eq:sub-ebb}
\text{Sub-EBB}\ \theta(\boldsymbol{x}) :\ &\boldsymbol{o}=\operatorname{bconv}_1(\boldsymbol{x}')+\gamma^{\text{L}}_1 f(\boldsymbol{x}'), \\
&\boldsymbol{x}'=\operatorname{bconv}_3(\boldsymbol{x})+\gamma^{\text{L}}_2 f(\boldsymbol{x}), \\
&\text{s.t.}\ \operatorname{gcd}(c_{\boldsymbol{x}}, c_{\boldsymbol{o}}) = c_{\boldsymbol{x}} \text{ or } c_{\boldsymbol{o}}, \\
\end{aligned}
\end{equation}
where $\gamma^\textit{L}$ denotes the 0-dimension scaling factor for layer-wise shortcuts that is learnable during training, $\operatorname{gcd}(\cdot, \cdot)$ denotes the greatest common divisor function, and $f(\cdot)$ denotes the mapping function, which can be defined as follow according to the channel numbers of input and output: 
\begin{equation}
f(\boldsymbol{x})=\left\{
\begin{aligned}
\frac{c_{\boldsymbol{o}}}{c_{\boldsymbol{x}}}\sum_{i=1}^{c_{\boldsymbol{x}}/c_{\boldsymbol{o}}} & \boldsymbol{x}^{\left\{(i-1) c_{\boldsymbol{o}}, i\cdot c_{\boldsymbol{o}}-1\right\}}, & \text{ if } c_{\boldsymbol{x}} > c_{\boldsymbol{o}}, \\
 \prod^{c_{\boldsymbol{o}}/c_{\boldsymbol{x}}} & \boldsymbol{x}, & \text{ otherwise},
\end{aligned}
\right.
\end{equation}
where $\boldsymbol{x}^{\left\{m: n\right\}}$ denotes taking the $m$ to $n$ dimension of $\boldsymbol x$, $\sum$ denotes the summation of all divided features, and $\prod$ denotes the $\frac{c_{\boldsymbol{o}}}{c_{\boldsymbol{x}}}$-time repeated concatenation.
The constructed sub-block allows autonomous modification, \textit{i.e.}, channel expansion, retention, or reduction of the output, corresponding $\theta^\uparrow$, $\theta^-$, and $\theta^\downarrow$ variants of sub-blocks, respectively, enabling flexible computation of feature extraction in the high- and low-dimension space, which is crucial in building the adaptable binarized architecture. 
The number of input and output channels of the feature shortcut can be aligned with the corresponding binarized convolutional layer to preserve representations related to the prediction information. The transformation allows flexible changes of feature channels through parameter-free operations, such as channel splitting, averaging, and concatenation, to create shortcuts with minimal computation overhead.

Each EBB consists of three well-designed sub-blocks: the head, the middle, and the tail. The head is responsible for receiving low-dimensional input, extracting features at the initial lower input dimension, or increasing the feature dimension preliminarily. The middle sub-block aims to extract features in a high-dimensional space and further expand the output channels, playing a crucial role in extracting rich representational information. The tail sub-block compresses the feature dimensions to prevent the continuous rise in feature dimensions from causing computational overhead and compresses the representational information. The EBB can be represented as follows:
\begin{equation}
\begin{aligned}
\phi(\boldsymbol{x}): &\boldsymbol{o}=\theta^\downarrow(\theta^\uparrow (\boldsymbol{x}')+ \gamma^\text{C}_2 f(\boldsymbol{x}))+ \gamma^\text{C}_1 f(\boldsymbol{x}'),\\ 
&\boldsymbol{x}'=\theta^{-}(\boldsymbol{x})[c_{\boldsymbol{x}}=c_{\boldsymbol{o}}]+\theta^\uparrow \left(\boldsymbol{x}\right)\left[c_{\boldsymbol{x}}\neq c_{\boldsymbol{o}}\right],
\end{aligned}
\end{equation}
where $\gamma^\textit{C}$ denotes the scaling factor for cross-layer-wise shortcuts and $[\cdot]$ represents the \textit{Iverson bracket}~\cite{iverson1962programming}.
In the EBB, we introduce cross-layer shortcuts crossing the sub-blocks. This mitigates the impact of multiple division and averaging processes on the original representation in layer-wise shortcuts and reduces the effect of nominal error accumulation caused by multiple layers of binarization.
When applying EBB, we also construct the block-wise shortcuts as follows:
\begin{equation}
\label{eq:ebb}
\text{EBB: } \psi(\boldsymbol{x}; \gamma^\text{L}, \gamma^\text{C}, \gamma^\text{B}): \boldsymbol{o} = \phi(\boldsymbol{x}) + \gamma^\text{B} f(\boldsymbol{x}),
\end{equation}
where $\gamma^\text{B}$ denotes the scaling factor of block-wise shortcuts.

Based on the proposed structure, we further construct its evolution strategy during training. Motivated by retaining prediction-relevant information, the EBB introduces layer-wise, cross-layer-wise, and block-wise shortcuts to avoid the effects of extreme binarization. However, dense information direct connections imply neglecting the role of computational units in information extraction and compression, which, in extreme cases, can even lead to the degradation of the backbone. Therefore, during training, we constrain the weights of the shortcuts' associated scaling factors using the following regularization method:
\begin{equation}
\label{eq:l_ebb}
\min\mathcal{L}_\text{EBB} = \sum\limits_{i=1}^{B} \boldsymbol{\gamma}_i^\text{C} + \boldsymbol{\gamma}_i^\text{B}.
\end{equation}
At the beginning of training, we initialize the scaling factors associated with the per-layer shortcuts to 1 and the cross-layer and block-wise shortcuts to $r \times 10^{-3}$, where $r$ is a random number between 0 and 1.
This constraint aims to minimize the degradation of computational units caused by connections crossing multiple layers while flexibly optimizing the backbone's information flow through learning. And the EBB structures $\psi(\boldsymbol{x}; \gamma^{\text{L}^*}, \gamma^{\text{C}^*}, \gamma^{\text{B}^*})$ with optimized $\gamma^{\text{L}^*}$, $\gamma^{\text{C}^*}$, and $\gamma^{\text{B}^*}$ parameters are applied during inference.

\subsection{BiVM: Sparse Heterogeneous Binarization}

\subsubsection{Computational Decoder Redundancy}
\label{subsubsec:ComputationalDecoderRedundancy}
In Section~\ref{subsec:BottlenecksofBinarized}, we observed that although binarizing the decoder does not severely affect accuracy as much as binarizing the encoder, the decoder generates a substantial computational overhead (more than 71.8\% of the entire network) with a small number of parameter (less than 12.1\%). Therefore, the computational burden of the decoder becomes the efficiency bottleneck for the binarized video matting network. This issue mainly stems from the computational demands of processing high-resolution features in the decoder. For example, consider the output block of the decoder, which uses two 3$\times$3 convolutions to process features at the original scale. Since grouped convolutions, which are detrimental to binarization, are not used here, the accuracy degradation due to decoder binarization is less significant than in the encoder. However, the computational load of this single block in the decoder (the last of five decoder blocks) is equivalent to 103\% of the entire encoder in the binarized baseline.

It is worth noting that not all spatial computations are equally important. In video matting, most frames contain large, continuous foreground and background regions that can often be identified at lower resolutions. This indicates that the involved computations are sparse, with great room for compression. Moreover, this sparsity is feature-related, making it difficult to alleviate through structurally reducing or compressing computational units in the network. The current decoder architecture repeats these computations at multiple scales, leading to enormous computational overhead, especially for high-resolution feature maps. This inefficiency poses a challenge for constructing an effective decoder, even for the most extreme 1-bit binarization.

Therefore, there is a need to fundamentally redesign computational units to develop an efficient decoder that reduces redundant computations while maintaining accurate matting results. This redesign should consider binarization and aim to simplify the computational process without compromising output quality.

\subsubsection{Sparse Heterogeneous Binarization for Efficient Decoder}
To address the computational redundancy of the decoder, especially when processing high-resolution features, we propose the Sparse Heterogeneous Binarization (SHB) method in BiVM (Figure~\ref{fig:overview}). This method aims to mitigate the dense computation over continuous regions of features in the decoder by improving the information heterogeneity of the binarized features through sparsification.

Inspired by the sparse segmentation of trimap and the classification of spatial regions, our hypothesis suggests optimizing the computation allocation in the matting decoder. Computations should not be uniformly distributed across the entire image's foreground/background areas but rather focused on the unknown alpha matting regions. To approximate these unknown regions and maximize computational savings during inference, we use the smallest scale features to generate an incoherent region binary mask~\cite{transfiner} $\boldsymbol{m}_\text{inc}$, sparsifying and heterogenizing the features involved in computation through information-guided processes. The specific computation method for $\boldsymbol{m}_\text{inc}$ is as follows:
\begin{equation}
\label{eq:shb-mask}
\boldsymbol{m}_\text{inc}^{\tau}=\operatorname{bool}\left(\left|\boldsymbol{f}_{\frac{N}{16}} - \mathcal{S}_{\uparrow}\left(\mathcal{S}_{\downarrow}\left(\boldsymbol{f}_{\frac{N}{16}}\right)\right)\right| - \tau\right),
\end{equation}
\begin{equation}
\label{eq:bool}
\operatorname{bool}(\boldsymbol{x})=
\begin{cases}
1,& \mathrm{if} \ \boldsymbol{x} \ge 0,\\
0,& \mathrm{otherwise.}
\end{cases}
\end{equation}
This computation aims to utilize the low-resolution features $\boldsymbol{f}_{\frac{N}{16}}$ obtained from the first decoder block to create a mask $\boldsymbol{m}_\text{inc} \in \{0, 1\}^{\frac{N}{16} \times \frac{N}{16}}$. In Eq.~\eqref{eq:shb-mask}, $N$ represents the original input scale, $\operatorname{bool}(\cdot)$ is the boolean function, and $\mathcal{S}_{\uparrow}(\cdot)$ and $\mathcal{S}_{\downarrow}(\cdot)$ are 2$\times$ up-sampling and down-sampling functions, respectively.
Since the original and sampled features are floating-point numbers, a threshold $\tau$ is introduced to divide the features before and after sampling and then use the boolean function to determine the mask value.

To achieve a balance between representation information and computational burden, we obtained optimized $\tau^*$ by
\begin{equation}
\label{eq:ie}
\max\limits_{\tau}\mathcal{H}\left(\operatorname{sign}\left(\operatorname{M}\left(\boldsymbol{f}_{\frac{N}{16}}; \boldsymbol{m}_\text{inc}^{\tau}\right)\right)\right)=-\sum\limits_{b\in\{-1,1\}}p(b)\ln\left(p(b)\right),
\end{equation}
where $\operatorname{M}\left(\boldsymbol{f}_{\frac{N}{16}}; \boldsymbol{m}_\text{inc}^{\tau}\right)$ denotes the masking operation guided by $\boldsymbol{m}_\text{inc}^{\tau}$, selecting the elements of $\boldsymbol{f}_{\frac{N}{16}}$ at positions where $\boldsymbol{m}_\text{inc}^{\tau} = 1$, $b\in\{-1, +1\}$ is a random variable obeying Bernoulli distribution sampled from $\operatorname{sign}\left(\operatorname{M}\left(\boldsymbol{f}_{\frac{N}{16}}; \boldsymbol{m}_\text{inc}^{\tau}\right)\right)$, and $\mathcal{H}(\cdot)$ denotes the information entropy. This optimization aims to ensure the heterogeneity of features after binarization by maximizing information entropy while sparsifying the features.
The obtained optimized threshold $\tau^*$ determines the binary mask $\boldsymbol{m}_\text{inc}^{\tau^*}$ used in the proposed SHB.
The masked incoherent regions are primarily located at the boundaries of object instances or high-frequency areas (the blue grid of sparse features in Figure~\ref{fig:overview}). The points within these incoherent regions are considered uncertain in alpha estimation, and the information within them is highly heterogeneous, thus containing very little redundancy. Computation can then ignore or simplify the remaining regions outside these incoherent areas. Since the features in this process constitute the lowest resolution among the four sets input to the decoder, the computation can be efficient with few burdens.

Subsequently, we utilize the identified mask $\boldsymbol{m}_\text{inc}$ to convert convolutions in high-resolution blocks into sparse-assisted binarized convolutions, ensuring efficient computation. Our devised masked sparse convolution differs from~\cite{transfiner,vmt} in that it doesn't perform pointwise attention for the entire incoherent regions, resulting in a design more conducive to binarization~\cite{qin2023bibench} and computation efficiency. Specifically, we modify the 3$\times$3 convolution in the up-sample block of the decoder to:
\begin{equation}
\label{eq:SHB}
\text{SHB: }
\boldsymbol{o} = \operatorname{SH-bconv}_3(\boldsymbol{x}; \operatorname{bilinear}^k(\boldsymbol{m}_\text{inc}^{\tau^*}))+\operatorname{bconv}_1(\boldsymbol{x}),
\end{equation}
where $\operatorname{SH-bconv}_3 (\cdot ; \boldsymbol{m}_\text{inc}^{\tau^*})$ denotes the binarized 3$\times$3 convolution assisted by a sparse mask $\boldsymbol{m}_\text{inc}^{\tau^*}$, where both the weight and sparse activation undergo binarization. Additionally, $\operatorname{bilinear}^k$ represents a bilinear interpolation up-sampling operation, where the low-resolution sparse mask $\boldsymbol{m}_\text{inc}$ is up-sampled by a factor of $k$, where $k \in \{2, 4, 8, 16\}$. 
For the binarized convolution $\operatorname{SH-bconv}_3$, we implement sparse computation following~\cite{tian2023designing,transfiner}, skipping regions with $\boldsymbol{m}_\text{inc}(x, y) = 0$ during inference. Importantly, this approach preserves this operation's non-grouped 3$\times$3 convolution nature, ensuring its practical binarization-friendliness.
Furthermore, in Eq.~(\ref{eq:SHB}), in addition to the sparse-assisted binarized 3$\times$3 convolution, we incorporate a 1$\times$1 binarized convolution layer to process feature extraction across the entire feature set. The output is fused with $\operatorname{SH-bconv}$ to guide finer predictions for continuous regions.

Our SHB first predicts a sparse binary mask based on low-resolution features, assisting binarization in higher-resolution features of the decoder. As demonstrated in the experimental section (Tab.~\ref{tab:ablation}), SHB significantly decreases the computational FLOPs of the binarized video matting network by significant 52.6\% while maintaining video matting accuracy.

\subsection{BiVM: Localized Binarization-aware Mimicking}

\subsubsection{Binarization-Aware Training Pipeline}
\label{subsec:TrainingPipeline}
We introduce the training pipeline for our BiVM network, incorporating additional training steps and iterations to ensure the comprehensive convergence of the binarized video matting network. This stands in contrast to the training of the full-precision RVM counterpart~\cite{lin2022robust}. Our BiVM pipeline consists of two phases: pre-training and matting training.

We conducted 200 training epochs on the binarized EBB backbone during the pre-training phase using the ImageNet classification dataset. This aimed to acquire a well-pre-trained backbone, facilitating smoother convergence in the subsequent matting training phase. Direct binarization of fully pre-trained models like MobileNetV3 can lead to near crashes. Hence, we apply the complete pre-training phase to all compared binarized video matting networks.

Then, the training phase for matting is sensitive to binarization and is constructed in accordance with~\cite{lin2022robust}, organized into 4 stages. \textit{Stage 1} entails undergoing 20 epochs of training on the low-resolution VM dataset without DGF, using $T=15$ frames for rapid updates. The learning rate for the EBB backbone is fixed at 1e-4, while the others are set to 2e-4. Furthermore, the input resolution $h, w$ is randomly selected from 256-512px to enhance resilience. In \textit{stage 2}, the network undergoes training with $T=50$, featuring a halved learning rate and an additional 2 epochs to facilitate the acquisition of long-term dependencies. During \textit{stage 3}, the DGF module is introduced, and a single epoch is dedicated to training on both low-resolution lengthy and high-resolution short sequences from the VM dataset. The low-resolution pass employs $T=40$, with $h$ and $w$ remaining the same as in stage 1 without DGF. Meanwhile, the high-resolution pass utilizes DGF with a downsampling factor $s=0.25$, $\hat{T}=6$, and $\hat{h}, \hat{w} \sim(1024,2048)$. The learning rate for DGF is 2e-4, and for the rest, it is 1e-5. In \textit{stage 4}, the network undergoes a 5-epoch training on D646 and AIM, with an increased learning rate for the decoder to 5e-5.

As for the loss function during training of the binarized video matting network, we follow the settings in~\cite{lin2022robust,qin2024bimatting}. 
For the per-training phase, the loss function can be presented as:
\begin{equation}
\mathcal{L}_\text{pre-training}=\mathcal{L}_\text{CE}+\lambda_\text{EBB}\mathcal{L}_{\text {EBB}},
\end{equation}
where $\mathcal{L}_\text{CE}$ denotes the cross-entropy loss~\cite{de2005tutorial}, $\mathcal{L}_{\text {EBB}}$ is as in Eq.~\eqref{eq:l_ebb}, and $\lambda_\text{EBB}$ represents a hyperparameter and is set as 1e-4 by default in our BiVM.
For the video matting training phase, the total matting loss $\mathcal{L}^M$ on all $t \in[1, T]$ frames is
\begin{equation}
\label{eq:l_matting}
\mathcal{L}^\text{M}=\mathcal{L}_{\ell 1}^\alpha+\mathcal{L}_{\text {lap }}^\alpha+5 \mathcal{L}_\text{t c}^\alpha+\mathcal{L}_{\ell 1}^F+5 \mathcal{L}_\text{t c}^F,
\end{equation}
where $\mathcal{L}_{l 1}^\alpha$ and $\mathcal{L}_{\text {lap }}^\alpha$ respectively denote $\ell$1 and pyramid Laplacian losses to learn alpha $\alpha_t$ with respect to the ground-truth $\alpha_t^*$~\cite{forte2020f,hou2019context}, temporal coherence loss $\mathcal{L}_{t c}^\alpha$~\cite{sun2021deep},
$\mathcal{L}_{\ell 1}^F$ and $\mathcal{L}_{t c}^F$ respectively denote the $\ell$1 and temporal coherence losses on pixels to learn foreground $F_t$ with respect to the ground-truth $F_t^*$, $\alpha_t^*>0$ following~\cite{lin2021real}.
For semantic segmentation, we use the binary cross entropy loss to learn the segmentation probability $S_\text{t}$ for the ground-truth binary label $S_\text{t}^*$~\cite{lin2022robust}.

\subsubsection{Localized Binarization-aware Mimicking for Stable Convergence}
Despite the benefits of a carefully designed neural architecture, the binarized video matting network struggles to converge as stably as its full-precision counterpart in a standard training pipeline due to the representation and gradient errors caused by the binarization function. An intuitive improvement is to introduce additional supervision to enhance optimization. Given that the full-precision counterpart has the same architecture and a stronger representational capacity, its features can be considered an ideal reference for the binarized network. Therefore, we introduce a Localized Binarization-aware Mimicking (LBM) technique into the standard optimization pipeline (Figure~\ref{fig:overview}). By mimicking the critical parts of the full-precision counterpart's highly relevant features to predict information, LBM improves the representations of the binarized video matting network and enables it to converge stably.

Specifically, we use the binary mask obtained in Eq.~\eqref{eq:ie} to focus distillation on locations on the information-dense incoherent regions so that the corresponding features of these regions are aligned with the full-precision counterpart for accurate prediction. The loss function can be expressed as
\begin{equation}
\mathcal{L}_\text{LBM}=\sum\limits_{\boldsymbol{f} \in \mathcal{F}}\left\|\boldsymbol{m}_\text{inc}^{\tau^*}\left(\frac{\boldsymbol{f}^2}{\left\|\boldsymbol{f}^2\right\|_2}-\frac{\boldsymbol{f}^2_\text{FP}}{\left\|\boldsymbol{f}^2_\text{FP}\right\|_2}\right)\right\|_2,
\end{equation}
where $\mathcal{F}=\left\{\boldsymbol{f}_{\frac{1}{2}}, \boldsymbol{f}_{\frac{1}{4}}, \boldsymbol{f}_{\frac{1}{8}}, \boldsymbol{f}_{\frac{1}{16}}\right\}$ denotes the set of different scale of features, $\mathcal{F}$ denote the feature sets from the encoder, and $\|\cdot\|_2$ denotes the $\ell$2 normalization. The total losses of matting training are
\begin{equation}
\mathcal{L}_\text{matting}=\mathcal{L}^\text{M}+\lambda_\text{LBM}\mathcal{L}_{\text{LBM}},
\end{equation}
where $\mathcal{L}^\text{M}$ is as in Eq.~\eqref{eq:l_matting} and $\lambda_\text{LBM}$ is set as 1e-4 by default.

The proposed LBM technique is performed in the first stage of the matting training process to make the initial learning process of the binarized video matting network, which focuses on critical features and converges stably.

\begin{table*}[t]
    \small
    \centering
    \setlength{\tabcolsep}{1.75mm}{
    \begin{tabular}{llllrrrrrrrrr}
        \toprule
        Model & Encoder & Decoder & Optim. & \#Bit & \#FLOPs\tiny{(G)} & \#Param\tiny{(MB)} & Latency\tiny{(ms)} & MAD & MSE & Grad & Conn & dtSSD \\ \midrule
        RVM & Vanilla\tiny{(MBV3)} & Vanilla & Vanilla & 32 & 4.57 & 14.5 & 921.42 & 6.08 & 1.47 & 0.88 & 0.41 & 1.36 \\
         & Vanilla & Vanilla & Vanilla & 1 & 0.55 & {0.64} & 60.65 & 28.49 & 18.16 & 6.80 & 3.74 & 3.64 \\ 
         & Vanilla & Vanilla & LBM & 1 & 0.55 & 0.64 & 60.65 & 24.65 & 15.79 & 4.68 & 3.16 & 2.56 \\\cmidrule{2-13}
         
         & EBB & Vanilla & Vanilla & 32 & 7.11 & 15.9 & 958.86 & 5.91 & 1.09 & 0.94 & 0.38 & 1.39  \\
         & EBB & Vanilla & Vanilla & 1 & 0.57 & 0.67 & 67.92 & 13.46 & 6.91 & 3.18 & 1.52 & 2.74 \\ \cmidrule{2-13}
         
         & Vanilla & SHB & Vanilla & 32 & 2.58 & 14.6 & 707.68 & 9.69 & 4.38 & 2.05 & 0.96 & 1.69 \\ 
         & Vanilla & SHB & Vanilla & 1 & 0.27 & 0.67 & 43.54 & 191.89 & 187.07 & 15.15 & 27.80 & 3.67  \\ \midrule
         
        BiVM & EBB & SHB & LBM & 32 & 5.02 & 16.3 & 745.11 &  7.43 & 2.48 & 1.45 & 0.62 & 1.57 \\
         & EBB & SHB & LBM  & 1 & 0.32 & 0.67 & 63.81 & 11.82 & 5.75 & 2.91 & 1.28 & 2.63 \\ \bottomrule
    \end{tabular}}
    \caption{Ablation result of BiVM on VM~\cite{lin2021real} dataset.}
    \label{tab:ablation}
\end{table*}

\begin{table*}[!t]
    \small
    \centering
    \setlength{\tabcolsep}{2.mm}{
    \begin{tabular}{llrrrrrrrrrr}
    \toprule
        ~ & ~ & ~ & ~ & ~ & ~ & ~ & Alpha & ~ & ~ & FG \\ 
        \cmidrule(lr){7-11}\cmidrule(lr){12-12}
        Dataset & Method & \#Bit & \#FLOPs\tiny{(G)} & \#Param\tiny{(MB)} & Latency\tiny{(ms)} & MAD  & MSE  & Grad  & Conn  & dtSSD  & MSE \\ \midrule
        VM & DeepLabV3~\cite{chen2017rethinking} & 32 & 136.06 & 223.66 & - & 14.47 & 9.67 & 8.55 & 1.69 & 5.18 & - \\ 
        \tiny{512$\times$288} & BGMv2~\cite{lin2021real} & 32 & 8.46 & 19.4 & - & 25.19 & 19.63 & 2.28 & 3.26 & 2.74 & - \\ 
        ~ & \textcolor{black}{MODNet~\cite{ke2022modnet}} & \textcolor{black}{32} & \textcolor{black}{8.80} & \textcolor{black}{25.0} & \textcolor{black}{-} & \textcolor{black}{9.41} & \textcolor{black}{4.30} & \textcolor{black}{1.89} & \textcolor{black}{0.81} & \textcolor{black}{2.23} & \textcolor{black}{-} \\ 
        ~ & RVM~\cite{lin2022robust} & 32 & 4.57 & 14.5 & 921.42 & 6.08 & 1.47 & 0.88 & 0.41 & 1.36 & - \\ \cmidrule{2-12}
        ~ & RVM-BNN$^\dagger$~\cite{hubara2016binarized} & 1 & 0.50 & 0.57 & 80.70 & 189.13 & 184.33 & 15.01 & 27.39 & 3.65 & - \\ 
        ~ & RVM-DoReFa~\cite{dorefa} & 1 & 0.52 & 0.64 & 97.66 & 51.64 &  34.50 &   8.85 &   7.14 &   4.09 & - \\ 
        ~ & RVM-ReCU$^\dagger$~\cite{xu2021recu} & 1 & 0.52 & 0.64 & 99.03 & 189.13 & 184.33 & 15.01 & 27.39 & 3.65 & - \\ 
        ~ & RVM-ReAct~\cite{liu2020reactnet} & 1 & 0.55 & 0.64 & 99.35 & 28.49 & 18.16 & 6.80 &   3.74 & 3.64 & - \\ 
        ~ & BiMatting~\cite{qin2024bimatting} & 1 & 0.37 & 0.67 & 65.76 & 12.82 & 6.65 & 2.97 & 1.42 & 2.69 & - \\ 
        ~ & BiVM (ours) & {1} & \textbf{0.32} & {0.67} & 63.81 & \textbf{11.82} & \textbf{5.75} & \textbf{2.91} & \textbf{1.28} & \textbf{2.63} & - \\ \midrule
        D646 & DeepLabV3~\cite{chen2017rethinking} & 32 & 241.89 & 223.66 & - & 24.50 & 20.1 & 20.30 & 6.41 & 4.51 & - \\ 
        \tiny{512$\times$512} & BGMv2~\cite{lin2021real} & 32 & 16.48 & 19.4 & - & 43.62 & 38.84 & 5.41 & 11.32 & 3.08 & 2.60 \\ 
        ~ & \textcolor{black}{MODNet~\cite{ke2022modnet}} & \textcolor{black}{32} & \textcolor{black}{15.64} & \textcolor{black}{25.0} & \textcolor{black}{-} & \textcolor{black}{10.62} & \textcolor{black}{5.71} & \textcolor{black}{3.35} & \textcolor{black}{2.45} & \textcolor{black}{1.57} & \textcolor{black}{6.31} \\ 
        ~ & RVM~\cite{lin2022robust} & 32 & 8.12 & 14.5 & 1313.55 & 7.28 & 3.01 & 2.81 & 1.83 & 1.01 & 2.93 \\ \cmidrule{2-12}
        ~ & RVM-BNN$^\dagger$~\cite{hubara2016binarized} & 1 & 0.88 & 0.57 & 149.81 & 281.20 & 276.85 &  25.26 &  73.59 &   1.08 &   6.95 \\ 
        ~ & RVM-DoReFa~\cite{dorefa} & 1 & 0.92 & 0.64 & 160.24 & 133.63 & 116.69 &  17.09 &  35.08 &   2.58 &   6.97  \\ 
        ~ & RVM-ReCU$^\dagger$~\cite{xu2021recu} & 1 & 0.92 & 0.64 & 159.35 & 281.20 & 276.85 &  25.26 &  73.59 &   1.08 &   6.95 \\ 
        ~ & RVM-ReAct~\cite{liu2020reactnet} & 1 & 0.97 & 0.64 & 169.75 & 56.41 &  43.10 &  14.05 &  14.85 &   2.56 &   6.85 \\ 
        ~ & BiMatting~\cite{qin2024bimatting} & 1 & 0.66 & {0.67} & 112.29 & {32.74} &  {24.48} &   {9.34} &   {8.62} &   {2.21} &   {5.86} \\ 
        ~ & BiVM (ours) & {1} & \textbf{0.57} & {0.67} & \textbf{102.98} & \textbf{31.12} & \textbf{23.68} & \textbf{8.32} & \textbf{8.13} & {2.15} & \textbf{5.26} \\\midrule
        AIM & DeepLabV3~\cite{chen2017rethinking} & 32 & 241.89 & 223.66 & - & 29.64 & 23.78 & 20.17 & 7.71 & 4.32 & - \\ 
        \tiny{512$\times$512} & BGMv2~\cite{lin2021real} & 32 & 16.48 & 19.4 & - & 44.61 & 39.08 & 5.54 & 11.60 & 2.69 & 3.31 \\ 
        ~ & \textcolor{black}{MODNet~\cite{ke2022modnet}} & \textcolor{black}{32} & \textcolor{black}{15.64} & \textcolor{black}{25.0} & \textcolor{black}{-} & \textcolor{black}{21.66} & \textcolor{black}{14.27} & \textcolor{black}{5.37} & \textcolor{black}{5.23} & \textcolor{black}{1.76} & \textcolor{black}{9.51} \\ 
        ~ & RVM~\cite{lin2022robust} & 32 & 8.12 & 14.5 & 1313.55 & 14.84 & 8.93 & 4.35 & 3.83 & 1.01 & 5.01 \\\cmidrule{2-12}
        ~ & RVM-BNN$^\dagger$~\cite{hubara2016binarized} & 1 & 0.88 &  0.57 & 149.81 & 327.02 & 321.15 &  23.80 &  85.55 &   0.75 &   7.84 \\  
        ~ & RVM-DoReFa~\cite{dorefa} & 1 & 0.92 & 0.64 & 160.24 & 129.29 & 107.79 &  17.31 &  34.18 &   2.62 &   7.85  \\ 
        ~ & RVM-ReCU$^\dagger$~\cite{xu2021recu} & 1 & 0.92 & 0.64 & 159.35 & 327.02 & 321.15 &  23.80 &  85.55 &   0.75 &   7.84 \\ 
        ~ & RVM-ReAct~\cite{liu2020reactnet} & 1 & 0.97 & 0.64 & 169.75 & 59.90 &  44.08 &  14.32 &  15.90 &   2.37 &   8.00 \\ 
        ~ & BiMatting~\cite{qin2024bimatting} & 1 & {0.66} & {0.67} & 112.29 & {35.17} &  {26.53} &   {9.42} &   {9.24} & {1.82} & {7.00} \\ 
        ~ & BiVM (ours) & {1} & \textbf{0.57} & {0.67} & \textbf{102.98} & \textbf{31.29} & \textbf{22.38} & \textbf{9.62} & \textbf{8.24} & 2.06 & \textbf{6.55} \\
        \bottomrule
    \end{tabular}}
    \caption{Low-resolution comparison on VM, D646, and AIM datasets. \textbf{Bold} indicates the best performance among 1-bit networks, and $^\dagger$ indicates the results are crashed.
    The 32-bit networks follow the released models or results directly~\cite{lin2022robust,BGMv2_github}.}
    \label{tab:main_result_lr}
\end{table*}

\begin{table*}[ht]
    \small
    \centering
    \setlength{\tabcolsep}{3.3mm}{
    \begin{tabular}{llrrrrrrrr}
    \toprule
        Dataset & Method & \#Bit & \#FLOPs\tiny{(G)} & \#Param\tiny{(MB)} & Latency\tiny{(ms)} & SAD  & MSE  & Grad & dtSSD \\ \midrule
        VM & BGMv2~\cite{lin2021real} & 32 & 9.86 & 19.4 & - & 49.83 &  44.71 &  74.71 &  4.09 \\ 
        \tiny{1920$\times$1080} & \textcolor{black}{MODNet + FGF~\cite{ke2022modnet}} & \textcolor{black}{32} & \textcolor{black}{7.78} & \textcolor{black}{25.0} & \textcolor{black}{-} & \textcolor{black}{11.13} & \textcolor{black}{5.54} & \textcolor{black}{15.30} & \textcolor{black}{3.08} \\
        ~ & RVM~\cite{lin2022robust} & 32 & 4.15 & 14.5 & 2497.79 & 6.57 & 1.93 & 10.55 & 1.90 \\ \cmidrule{2-10}
          & RVM-ReAct~\cite{liu2020reactnet} & 1 & 0.53 &  0.64 & 713.02 & 31.60 &  20.29 &  34.28  &  4.08  \\
        ~ & BiMatting~\cite{qin2024bimatting} & 1 & 0.38 & {0.67} & 556.24 & {18.16} & {11.15} & {21.90} & {3.25} \\ 
        ~ & BiVM (ours) & 1 & \textbf{0.33} & {0.67} & \textbf{481.31} & \textbf{14.54} & \textbf{8.19} & \textbf{21.66}  & \textbf{3.17}  \\ \midrule
        D646 & BGMv2~\cite{lin2021real} & 32 & 9.86 & 19.4 & - & 57.40 & 52.00 & 149.20 &   2.56 \\ 
        \tiny{1920$\times$1080} & \textcolor{black}{MODNet + FGF~\cite{ke2022modnet}} & \textcolor{black}{32} & \textcolor{black}{7.78} & \textcolor{black}{25.0} & \textcolor{black}{-} & \textcolor{black}{11.27} & \textcolor{black}{6.13} & \textcolor{black}{30.78} & \textcolor{black}{2.19} \\
        ~ & RVM~\cite{lin2022robust} & 32& 4.15 & 14.5 & 2497.79 & 8.67 & 4.28 & 30.06 & 1.64 \\ \cmidrule{2-10} 
          & RVM-ReAct~\cite{liu2020reactnet} & 1 & 0.53 &  0.64 & 713.02 & 57.38 & 42.14 &  71.24 &   3.03 \\
        ~ & BiMatting~\cite{qin2024bimatting} & 1 & {0.38} & {0.67} & 556.24 & {52.85} &  {44.08} &  {61.60} & {3.12} \\ 
        ~ & BiVM (ours) & 1 & \textbf{0.33} & {0.67} & \textbf{481.31} & \textbf{37.47} & \textbf{29.89} & \textbf{59.29} & \textbf{3.09}  \\ \midrule
        AIM & BGMv2~\cite{lin2021real} & 32 & 15.19 & 19.4 & - & 45.76 &  38.75 & 124.06 &   2.02 \\ 
        \tiny{2048$\times$2048} & \textcolor{black}{MODNet + FGF~\cite{ke2022modnet}} & \textcolor{black}{32} & \textcolor{black}{17.70} & \textcolor{black}{25.0} & \textcolor{black}{-} & \textcolor{black}{17.29} & \textcolor{black}{10.10} & \textcolor{black}{35.52} & \textcolor{black}{2.60} \\
        ~ & RVM~\cite{lin2022robust} & 32 & 8.37 & 14.5 & 4813.71 & 14.89 & 9.01 & 34.97 & 1.71 \\ \cmidrule{2-10} 
          & RVM-ReAct~\cite{liu2020reactnet} & 1 & 1.07 &  0.64 & 1798.14 & 57.38 &  42.14 &  71.24 & 3.03 \\
        ~ & BiMatting~\cite{qin2024bimatting} & 1 & 0.77 & {0.67} & 978.94 & 48.27 &  38.37 &  61.72 & 2.80 \\ 
        ~ & BiVM (ours) & 1 & \textbf{0.67} & {0.67} & \textbf{913.20} & \textbf{47.79} & \textbf{38.89} & \textbf{60.37} & \textbf{2.80} \\ \bottomrule
    \end{tabular}}
    \caption{High-resolution comparison on VM, D646, and AIM datasets.
    }
    \label{tab:main_result_hr}
\end{table*}

\section{Experiments}
We comprehensively evaluated the proposed BiVM across various benchmarks to present the performance.
\textbf{1)} We construct thorough ablation studies on the VM~\citep{lin2021real} dataset for the techniques in our proposed BiVM algorithm, \textit{i.e.}, to show their effectiveness on accuracy and efficiency.
\textbf{2)} We compare the proposed BiVM with 1-bit video matting networks employing advanced binarization algorithms~\cite{hubara2016binarized,dorefa,liu2020reactnet,xu2021recu} on VM~\citep{lin2021real}, D646~\cite{qiao2020attention}, and AIM~\cite{xu2017deep} datasets, showing BiVM can achieve the state-of-the-art (SOTA) accuracy results under binarization across different datasets. 
We excluded the use of Image Backgrounds used in RVM~\cite{lin2022robust} during training since it is inaccessible, and thus used only Video Backgrounds.
Notably, BiVM can even outperform certain existing full-precision video matting networks with up to hundreds of times parameters and FLOPs reductions. Remarkably, BiVM significantly diminishes computational FLOPs and model size by $14.3\times$ and $21.6\times$, respectively, showcasing its efficiency.
For the evaluation metrics, we apply Mean Absolute Difference (MAD), Mean Squared Error (MSE), spatial Gradient (Grad), and Connectivity (Conn) for evaluating $\alpha$ to measure quality, and the temporal coherence is evaluated using the dtSSD metric. We measure pixels where $\alpha>0$ via MSE for $F$~\cite{lin2022robust}.
\textbf{3)} We evaluate BiVM on real ARM CPU hardware, including mobile and edge devices, revealing the efficiency in deployment scenarios.
{The latency metrics in Tables~\ref{tab:ablation}, \ref{tab:main_result_lr}, and \ref{tab:main_result_hr} are based on the HUAWEI P40 Pro with a Kirin 990 processor, while Table~\ref{tab:hardware} further includes efficiency evaluation metrics based on the OnePlus 7T with a Snapdragon 855+ processor.}
\textbf{4)} We also present diverse visualization comparisons to demonstrate the effects of the proposed BiVM in practical applications.
We maintain a batch size of 4 throughout our experiments, distributed across 4 NVIDIA A800 GPUs for all training stages of BiVM.

\begin{figure*}[!t]
  \centering
  \includegraphics[width=\textwidth]{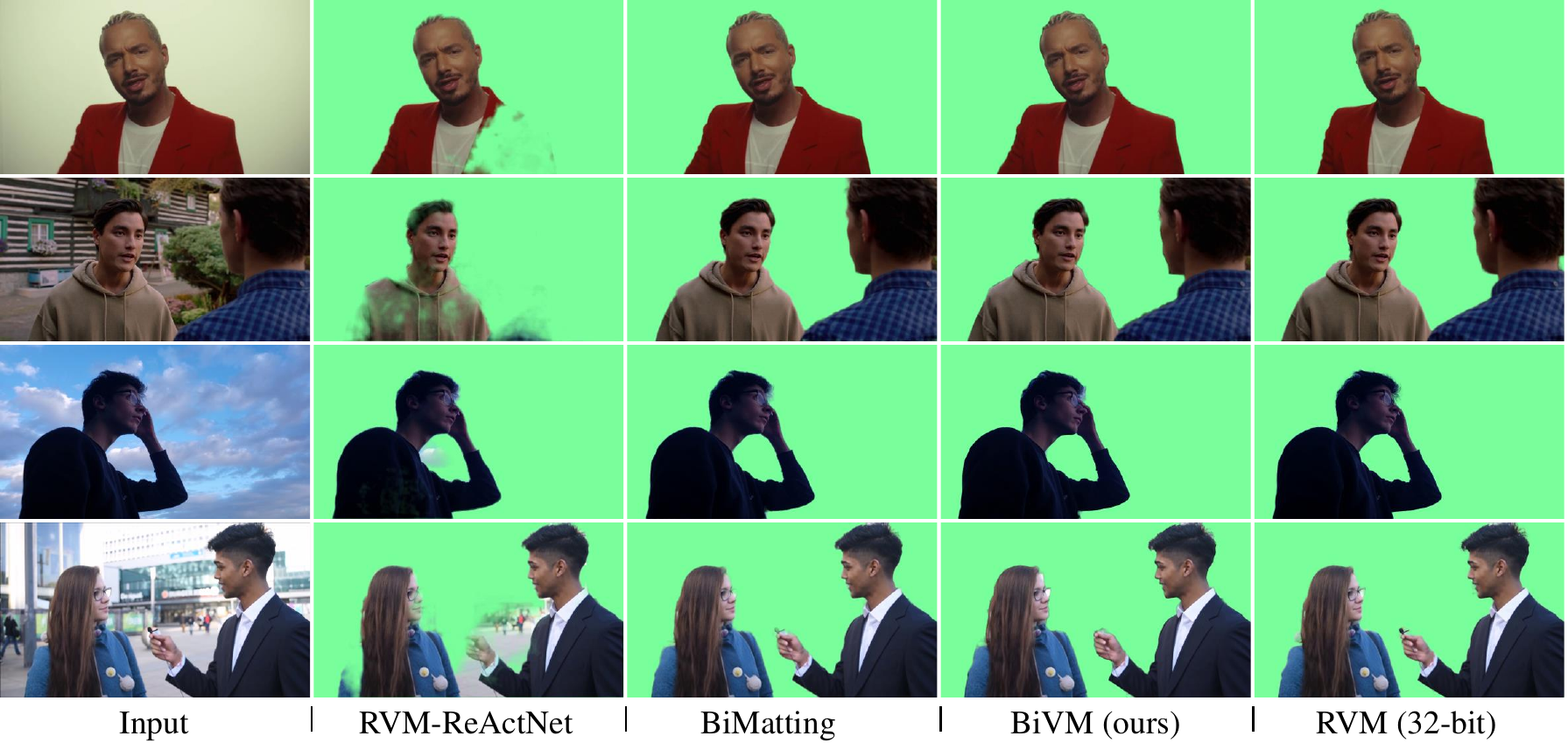}
  \caption{Frame matting comparison among 1-bit RVM-ReActNet, 1-bit BiMatting, 1-bit BiVM (ours), and 32-bit RVM models.}
  \label{fig:visual_overview}
\end{figure*}

\subsection{Quantitative Results}
\subsubsection{Ablation Study}
In Table~\ref{tab:ablation}, we ablate the techniques of BiVM on VM datasets in accuracy and efficiency.
We first compare all binarized video matting networks. For the binarized baseline, the accuracy degradation is evident across all metrics related to VM data recall. And despite a 22.7$\times$ parameter compression achieved by the baseline on efficiency, the binarized network only manages an 8.3$\times$ computational savings in efficiency. 
Then, we apply techniques in our BiVM to the binarized video matting network.
Substituting the encoder solely with our EBB noticeably restores the accuracy of the binarized video matting network, confirming the encoder as the primary limitation of accuracy in the baseline. 
LBM can also improve the accuracy of the binarized video matting network across various metrics as an optimization technique.
Even though the efficiency of the SHB is significantly improved when applied separately in the decoder, the accuracy of the binarized video matting network is still dangerously close to collapse. In contrast, a substantial accuracy improvement is achieved when the other two techniques are applied. 
The significant boost in accuracy performance in BiVM when combining all three enhancements emphasizes that the decoder should focus on fewer but essential representations to elevate accuracy by delivering high-quality features.

{Then, we compared the full-precision and binarized networks across different architectures under the same optimization pipeline. We found that, despite their differing architectures, all full-precision models exhibited a similar accuracy level, but their binarized counterparts showed significant variation. Among them, the binarized model using SHB alone suffered the largest drops in accuracy after binarization, up to 182.20 MAD. By contrast, applying only EBB reduces the accuracy degradation from binarization to just 7.55. This confirms the motivation we presented in Section~\ref{subsec:BottlenecksofBinarized}: the encoder of the original RVM is the main accuracy bottleneck in binarization. With the same LBM optimization, combining EBB and SHB further reduces binarization degradation to only 4.39 MAD (BiVM 11.82 \textit{vs.} full-precision 7.43), pushing the accuracy limits of binarized video matting networks.}

\begin{figure*}[!t]
  \centering
  \vspace{-0.05in}
  \includegraphics[width=\textwidth]{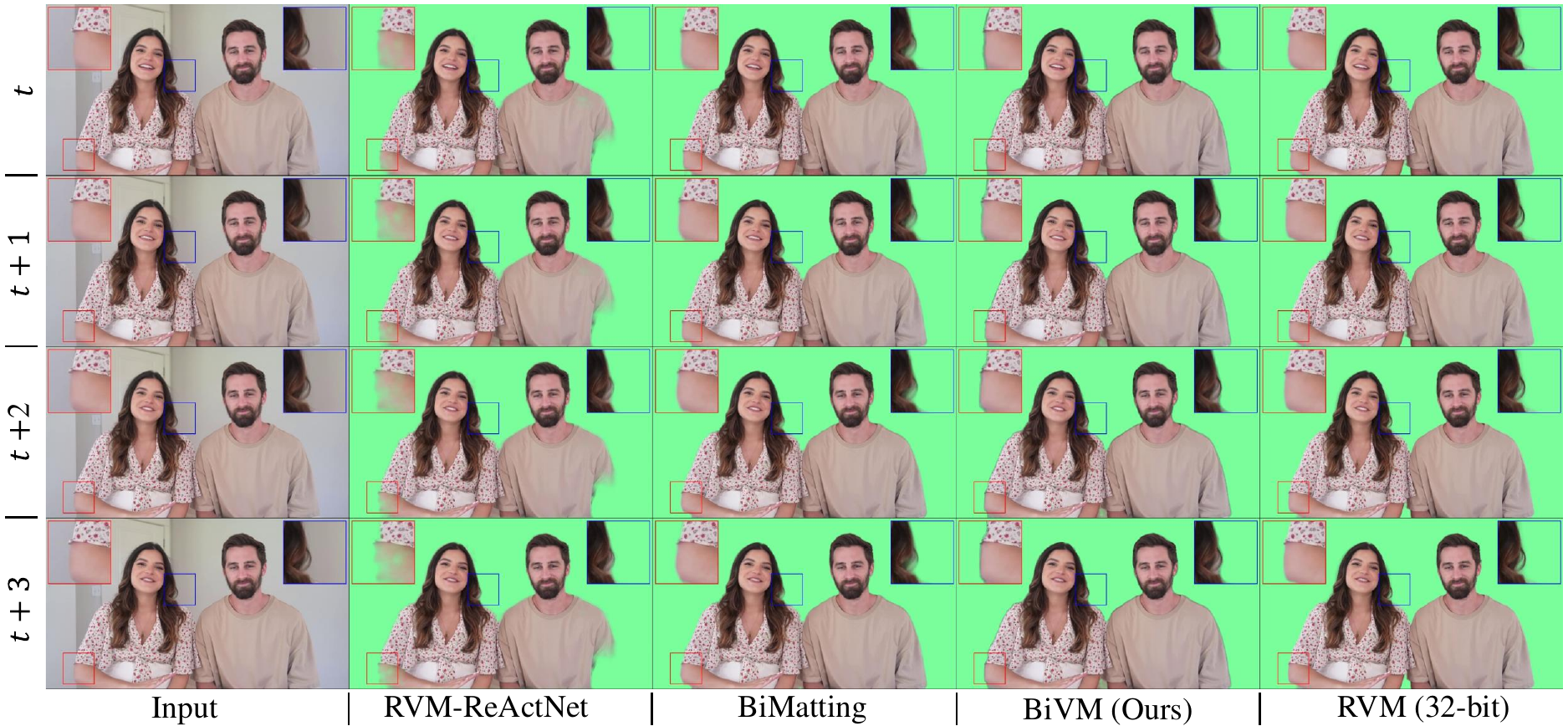}
  \vspace{-0.25in}
  \caption{Femporal coherence comparison among 1-bit RVM-ReActNet, 1-bit BiMatting, 1-bit BiVM (ours), and 32-bit RVM models.}
  \label{fig:visual_temporal}
\end{figure*}

\subsubsection{Main Comparison}
We use test data from the VM, D646, and AIM benchmarks to establish a comprehensive and fair comparison, integrating video backgrounds and image backgrounds following the configurations outlined in~\cite{lin2022robust,rvm_github}. Each testing clip contains 100 frames, wherein motion augmentation is applied to image samples.
Our evaluation compares our proposed BiVM network against the video matting networks binarized by existing binarization algorithms. The binarization algorithms includes classical BNN~\cite{phan2020binarizing} and DoReFa~\cite{dorefa}, as well as state-of-the-art (SOTA) ReActNet~\cite{liu2020reactnet} and ReCU~\cite{xu2021recu}, with the latter two regarded as exemplary practices for generic binarization~\cite{qin2023bibench}. 
\textcolor{black}{Results from some full-precision video matting methods are included for reference, such as RVM~\cite{lin2022robust} with MobileNetV3~\cite{howard2019searching} backbone, DeepLabV3~\cite{chen2017rethinking}, BGMv2~\cite{lin2021real} with MobileNetV2~\cite{sandler2018mobilenetv2} backbone, and MODNet~\cite{ke2022modnet} (+Fast Guided Filter (FGF)~\cite{lin2022robust}).}
For the fair comparison, we subject all binarized networks to the identical training pipeline as BiVM, as elaborated in Section~\ref{subsec:TrainingPipeline}. Regarding the full-precision video matting networks, we adhere to the results documented in prior studies~\cite{lin2022robust} unless otherwise specified.

Table~\ref{tab:main_result_lr} exhibits the comparison results of different binarized video matting networks on low-resolution input. The results signify that directly applying BNN and ReCU to the full-precision RVM yields a crash in accuracy for the binarized versions. This unexpected outcome is noteworthy, especially considering the latter method is classified among the leading binarization methods, underscoring the intricate nature of binarizing established video matting networks. In contrast, our BiVM model consistently outperforms all prevailing binarization models across diverse datasets. For example, compared with the network binarized by the ReActNet method, BiVM surpasses 16.67 on the MAD and 12.41 on the MSE metrics for alpha prediction. Note that the proposed BiVM also outperforms previous BiMatting~\cite{qin2024bimatting} in accuracy with greatly reduced FLOPs, showing the significant advantages of architecture and optimization designs. A detailed analysis is provided in Section~\ref{subsec:Visualization} with extensive visualizations. Moreover, BiVM surpasses even some 32-bit full-precision models when constrained to a mere 1-bit limit bit width. For example, BiVM outshines BGMv2 on VM, D646, and AIM datasets, and the latter even has hundreds of times the parameters and FLOPs compared with the binarized BiVM. It also outperforms DeepLabV3 on the VM dataset. These findings underscore the substantial potential of binarization in enhancing the efficiency of video matting networks. Table~\ref{tab:main_result_hr} compares our BiVM approach and other methodologies using high-resolution datasets. Our method consistently outperforms existing binarized video matting networks and even 32-bit BGMv2 across various datasets and metrics, affirming the robustness of improvements achieved by BiVM under diverse resolutions.

{Furthermore, BiVM exhibits notable efficiency alongside accurate prediction, showcasing its considerable potential in video matting. As illustrated in Table~\ref{tab:main_result_lr} and Table~\ref{tab:main_result_hr}, when compared to its full-precision counterpart, BiVM realizes a reduction in computational FLOPs by 14.3$\times$ and a decrease in parameters by 21.6$\times$, resulting in an inference latency low to just 7.8\% of the former on real-world hardware with 512$\times$512 input size (32-bit RVM 1313.55 ms \textit{vs.} 1-bit BiVM 102.98 ms). Note that BiVM also achieves an impressive 425$\times$ FLOPs saving against the full-precision DeepLabV3 model while achieving better accuracy performance. These advancements position BiVM as a highly prospective solution for deployment in edge computing, demonstrating the potential of the practical application in resource-constrained hardware.}

\subsection{Visualization Results}
\label{subsec:Visualization}

\begin{figure}[t]
\centering
\includegraphics[width=1.\textwidth]{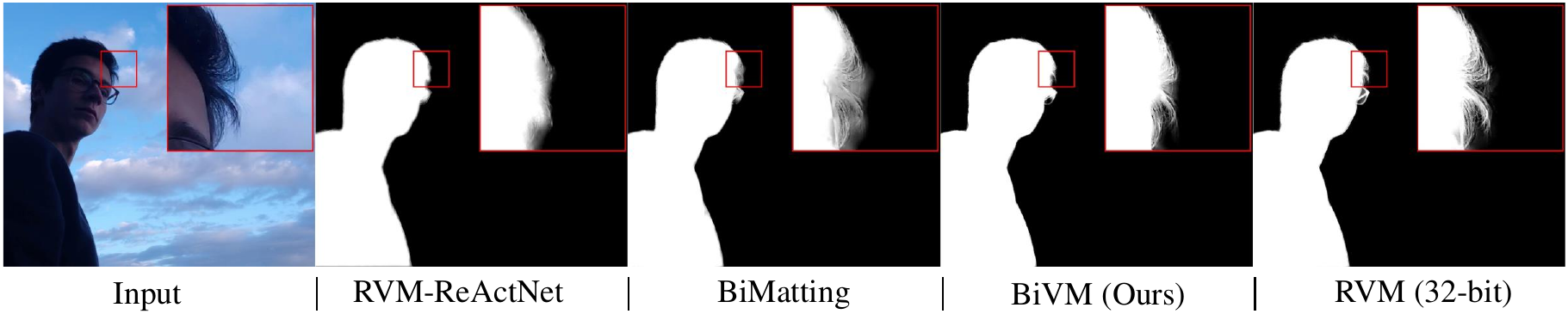}
\caption{Alpha detail comparison.}
\label{fig:alpha_details}
\end{figure}

Figure~\ref{fig:visual_overview} and Figure~\ref{fig:visual_temporal} showcase qualitative comparisons of natural videos. Figure~\ref{fig:visual_overview} contrasts RVM-ReActNet, BiMatting (the current state-of-the-art binarized video matting network), and the full-precision RVM with our BiVM model. The latter can be considered as the 32-bit counterpart of BiVM. In general, the resource consumption of BiVM is significantly reduced while maintaining performance close to the full-precision counterpart. Specifically, Figure~\ref{fig:visual_overview} presents experiments on videos from diverse frames, revealing the robustness of BiVM against semantic errors. Moreover, BiVM excels in matting edge regions compared to the other two models. Figure~\ref{fig:visual_temporal} focuses on temporal coherence. Our BiVM consistently outperforms in producing superior results, whereas RVM-ReActNet and BiMatting cause errors in several areas. Figure~\ref{fig:alpha_details} presents an alpha prediction comparison among various methods. Our method surpasses the SOTA binarized video matting network by accurately predicting intricate details, such as individual hair strands.

\begin{table}[t]
    \small
    \centering
    \setlength{\tabcolsep}{2.5mm}
    {
    \begin{tabular}{lrlr}
        \toprule
        Model & \#Bit & Device (CPU) & Latency\tiny{(ms)} \\ \midrule
        Full-Precision & 32 & {OnePlus 7T} & 2232.79 \\ 
        BiVM (ours) & 1 & (Snapdragon 855+) & 162.45 \\ 
        \midrule
        Full-Precision & 32 & {HUAWEI P40Pro} & 1313.55 \\ 
        BiVM (ours) & 1  & (Kirin 990) & 102.98 \\ 
        \bottomrule
    \end{tabular}}
    \caption{Deployment evaluation with 512$\times$512 input size.}
    \label{tab:hardware}
\end{table}

\subsection{Deployment Results}
To further confirm the effectiveness of BiVM in real-world deployment scenarios, we conducted evaluations on various mobile phones equipped with ARM CPUs, including HUAWEI P40 Pro with Kirin 990 and OnePlus 7T with Snapdragon 855+, measuring the actual processing speed in practice. We adopted the Larq~\cite{geiger2020larq} deployment library supporting binarized operators to ensure compatibility with BiVM. In Table~\ref{tab:hardware}, we compare BiVM and its full-precision counterpart, showing BiVM achieves 13.2$\times$ faster inference.

\subsection{{Potential Limitation}}
{Though BiVM achieves significant improvements in accuracy and efficiency over existing binarized methods, it still exhibits a slight performance gap compared to its full-precision counterpart, the 32-bit RVM, particularly in capturing highly intricate local details such as fine hair strands and translucent edges. The reason could be that the highly lightweight architecture may cause the model to occasionally overlook subtle foreground-background transitions in complex scenes, leading to minor artifacts in regions with high-frequency textures. This highlights opportunities for further enhancing the balance between extreme binarization and accurate spatial-temporal representation.}

\section{Conclusion}
Real-time video matting on edge devices faces significant computational challenges, while current compact binarized video matting networks still suffer from accuracy and efficiency issues. To address these, we introduce BiVM, a resource-efficient binarized neural network for video matting. BiVM features elastic shortcuts and evolvable topologies in its encoder, improving representation quality for accurate predictions. It also sparsifies the binarized decoder by masking redundant computations, enhancing inference efficiency. Additionally, a localized binarization-aware mimicking framework is implemented to utilize matting-related representation from full-precision counterparts. Extensive experiments demonstrate that BiVM outperforms state-of-the-art binarized video matting models, achieving visual quality comparable to full-precision models while significantly reducing computation and storage needs. Real-world tests on ARM CPU mobile devices further highlight the potential of BiVM in resource-constrained scenarios.

\ifCLASSOPTIONcompsoc
  \section*{Acknowledgments}
\else
  \section*{Acknowledgment}
\fi

This work was supported by the National Natural Science Foundation of China (Nos. 62476018), the Swiss National Science Foundation (SNSF) project 200021E\_219943 Neuromorphic Attention Models for Event Data (NAMED), and the Beijing Municipal Science and Technology Project (No. Z231100010323002).

{ \small
\bibliography{egbib}

\begin{thebibliography}{100}\itemsep=-1pt

\bibitem{agustsson2020universally}
Eirikur Agustsson and Lucas Theis.
\newblock Universally quantized neural compression.
\newblock {\em Advances in neural information processing systems}, 33:12367--12376, 2020.

\bibitem{aksoy2017designing}
Yagiz Aksoy, Tunc Ozan~Aydin, and Marc Pollefeys.
\newblock Designing effective inter-pixel information flow for natural image matting.
\newblock In {\em Proceedings of the IEEE Conference on Computer Vision and Pattern Recognition}, pages 29--37, 2017.

\bibitem{bai2021binarybert}
Haoli Bai, Wei Zhang, Lu Hou, Lifeng Shang, Jin Jin, Xin Jiang, Qun Liu, Michael Lyu, and Irwin King.
\newblock Binarybert: Pushing the limit of bert quantization.
\newblock In {\em Annual Meeting of the Association for Computational Linguistics}, pages 4334--4348, 2021.

\bibitem{bai2007geodesic}
Xue Bai and Guillermo Sapiro.
\newblock A geodesic framework for fast interactive image and video segmentation and matting.
\newblock In {\em 2007 IEEE 11th International Conference on Computer Vision}, pages 1--8. IEEE, 2007.

\bibitem{bai2018proxquant}
Yu Bai, Yu-Xiang Wang, and Edo Liberty.
\newblock Proxquant: Quantized neural networks via proximal operators.
\newblock In {\em International Conference on Learning Representations}, 2018.

\bibitem{banner2018aciq}
Ron Banner, Yury Nahshan, Elad Hoffer, and Daniel Soudry.
\newblock Aciq: Analytical clipping for integer quantization of neural networks.
\newblock 2018.

\bibitem{bengio2013estimating}
Yoshua Bengio, Nicholas L{\'e}onard, and Aaron Courville.
\newblock Estimating or propagating gradients through stochastic neurons for conditional computation.
\newblock {\em arXiv preprint arXiv:1308.3432}, 2013.

\bibitem{bethge2020meliusnet}
Joseph Bethge, Christian Bartz, Haojin Yang, Ying Chen, and Christoph Meinel.
\newblock Meliusnet: Can binary neural networks achieve mobilenet-level accuracy?
\newblock {\em arXiv preprint arXiv:2001.05936}, 2020.

\bibitem{bulat2018hierarchical}
Adrian Bulat and Georgios Tzimiropoulos.
\newblock Hierarchical binary cnns for landmark localization with limited resources.
\newblock {\em IEEE transactions on pattern analysis and machine intelligence}, 42(2):343--356, 2018.

\bibitem{XNOR++}
Adrian Bulat and Georgios Tzimiropoulos.
\newblock Xnor-net++: Improved binary neural networks.
\newblock {\em CoRR}, abs/1909.13863, 2019.

\bibitem{cai2024binarized}
Yuanhao Cai, Yuxin Zheng, Jing Lin, Xin Yuan, Yulun Zhang, and Haoqian Wang.
\newblock Binarized spectral compressive imaging.
\newblock {\em Advances in Neural Information Processing Systems}, 36, 2024.

\bibitem{cai2017deep}
Zhaowei Cai, Xiaodong He, Jian Sun, and Nuno Vasconcelos.
\newblock Deep learning with low precision by half-wave gaussian quantization.
\newblock In {\em Proceedings of the IEEE conference on computer vision and pattern recognition}, pages 5918--5926, 2017.

\bibitem{chen2017deeplab}
Liang-Chieh Chen, George Papandreou, Iasonas Kokkinos, Kevin Murphy, and Alan~L Yuille.
\newblock Deeplab: Semantic image segmentation with deep convolutional nets, atrous convolution, and fully connected crfs.
\newblock {\em IEEE transactions on pattern analysis and machine intelligence}, 40(4):834--848, 2017.

\bibitem{chen2017rethinking}
Liang-Chieh Chen, George Papandreou, Florian Schroff, and Hartwig Adam.
\newblock Rethinking atrous convolution for semantic image segmentation.
\newblock {\em arXiv preprint arXiv:1706.05587}, 2017.

\bibitem{chen2018semantic}
Quan Chen, Tiezheng Ge, Yanyu Xu, Zhiqiang Zhang, Xinxin Yang, and Kun Gai.
\newblock Semantic human matting.
\newblock In {\em Proceedings of the 26th ACM international conference on Multimedia}, pages 618--626, 2018.

\bibitem{chen2013knn}
Qifeng Chen, Dingzeyu Li, and Chi-Keung Tang.
\newblock Knn matting.
\newblock {\em IEEE transactions on pattern analysis and machine intelligence}, 35(9):2175--2188, 2013.

\bibitem{choukroun2019low}
Yoni Choukroun, Eli Kravchik, Fan Yang, and Pavel Kisilev.
\newblock Low-bit quantization of neural networks for efficient inference.
\newblock In {\em 2019 IEEE/CVF International Conference on Computer Vision Workshop (ICCVW)}, pages 3009--3018. IEEE, 2019.

\bibitem{chuang2001bayesian}
Yung-Yu Chuang, Brian Curless, David~H Salesin, and Richard Szeliski.
\newblock A bayesian approach to digital matting.
\newblock In {\em Proceedings of the 2001 IEEE Computer Society Conference on Computer Vision and Pattern Recognition. CVPR 2001}, volume~2, pages II--II. IEEE, 2001.

\bibitem{courbariaux2015binaryconnect}
Matthieu Courbariaux, Yoshua Bengio, and Jean-Pierre David.
\newblock Binaryconnect: Training deep neural networks with binary weights during propagations.
\newblock {\em Advances in neural information processing systems}, 28, 2015.

\bibitem{de2005tutorial}
Pieter-Tjerk De~Boer, Dirk~P Kroese, Shie Mannor, and Reuven~Y Rubinstein.
\newblock A tutorial on the cross-entropy method.
\newblock {\em Annals of operations research}, 134:19--67, 2005.

\bibitem{Regularize-act-distribution}
Ruizhou Ding, Ting-Wu Chin, Zeye Liu, and Diana Marculescu.
\newblock Regularizing activation distribution for training binarized deep networks.
\newblock In {\em CVPR}, June 2019.

\bibitem{dong2019hawq}
Zhen Dong, Zhewei Yao, Amir Gholami, Michael~W Mahoney, and Kurt Keutzer.
\newblock Hawq: Hessian aware quantization of neural networks with mixed-precision.
\newblock In {\em Proceedings of the IEEE/CVF international conference on computer vision}, pages 293--302, 2019.

\bibitem{du2024bitdistiller}
Dayou Du, Yijia Zhang, Shijie Cao, Jiaqi Guo, Ting Cao, Xiaowen Chu, and Ningyi Xu.
\newblock Bitdistiller: Unleashing the potential of sub-4-bit llms via self-distillation.
\newblock {\em arXiv preprint arXiv:2402.10631}, 2024.

\bibitem{fan2020training}
Angela Fan, Pierre Stock, Benjamin Graham, Edouard Grave, R{\'e}mi Gribonval, Herve Jegou, and Armand Joulin.
\newblock Training with quantization noise for extreme model compression.
\newblock {\em arXiv preprint arXiv:2004.07320}, 2020.

\bibitem{fan2011scribble}
Jialue Fan, Xiaohui Shen, and Ying Wu.
\newblock Scribble tracker: a matting-based approach for robust tracking.
\newblock {\em IEEE Transactions on Pattern Analysis and Machine Intelligence}, 34(8):1633--1644, 2011.

\bibitem{feng2016cluster}
Xiaoxue Feng, Xiaohui Liang, and Zili Zhang.
\newblock A cluster sampling method for image matting via sparse coding.
\newblock In {\em Computer Vision--ECCV 2016: 14th European Conference, Amsterdam, The Netherlands, October 11-14, 2016, Proceedings, Part II 14}, pages 204--219. Springer, 2016.

\bibitem{forte2020f}
Marco Forte and Fran{\c{c}}ois Piti{\'e}.
\newblock $ f $, $ b $, alpha matting.
\newblock {\em arXiv preprint arXiv:2003.07711}, 2020.

\bibitem{gastal2010shared}
Eduardo~SL Gastal and Manuel~M Oliveira.
\newblock Shared sampling for real-time alpha matting.
\newblock In {\em Computer Graphics Forum}, volume~29, pages 575--584. Wiley Online Library, 2010.

\bibitem{geiger2020larq}
Lukas Geiger and Plumerai Team.
\newblock Larq: An open-source library for training binarized neural networks.
\newblock {\em Journal of Open Source Software}, 5(45):1746, 2020.

\bibitem{gholami2022survey}
Amir Gholami, Sehoon Kim, Zhen Dong, Zhewei Yao, Michael~W Mahoney, and Kurt Keutzer.
\newblock A survey of quantization methods for efficient neural network inference.
\newblock In {\em Low-Power Computer Vision}, pages 291--326. Chapman and Hall/CRC, 2022.

\bibitem{grady2005random}
Leo Grady, Thomas Schiwietz, Shmuel Aharon, and R{\"u}diger Westermann.
\newblock Random walks for interactive alpha-matting.
\newblock In {\em Proceedings of VIIP}, volume 2005, pages 423--429. Citeseer, 2005.

\bibitem{gu2022factormatte}
Zeqi Gu, Wenqi Xian, Noah Snavely, and Abe Davis.
\newblock Factormatte: Redefining video matting for re-composition tasks.
\newblock {\em arXiv preprint arXiv:2211.02145}, 2022.

\bibitem{he2011global}
Kaiming He, Christoph Rhemann, Carsten Rother, Xiaoou Tang, and Jian Sun.
\newblock A global sampling method for alpha matting.
\newblock In {\em CVPR 2011}, pages 2049--2056. Ieee, 2011.

\bibitem{heo2022vita}
Miran Heo, Sukjun Hwang, Seoung~Wug Oh, Joon-Young Lee, and Seon~Joo Kim.
\newblock Vita: Video instance segmentation via object token association.
\newblock {\em Advances in Neural Information Processing Systems}, 35:23109--23120, 2022.

\bibitem{hou2019context}
Qiqi Hou and Feng Liu.
\newblock Context-aware image matting for simultaneous foreground and alpha estimation.
\newblock In {\em Proceedings of the IEEE/CVF International Conference on Computer Vision}, pages 4130--4139, 2019.

\bibitem{howard2019searching}
Andrew Howard, Mark Sandler, Grace Chu, Liang-Chieh Chen, Bo Chen, Mingxing Tan, Weijun Wang, Yukun Zhu, Ruoming Pang, Vijay Vasudevan, et~al.
\newblock Searching for mobilenetv3.
\newblock In {\em Proceedings of the IEEE/CVF international conference on computer vision}, pages 1314--1324, 2019.

\bibitem{huang2024billm}
Wei Huang, Yangdong Liu, Haotong Qin, Ying Li, Shiming Zhang, Xianglong Liu, Michele Magno, and Xiaojuan Qi.
\newblock Billm: Pushing the limit of post-training quantization for llms.
\newblock {\em arXiv preprint arXiv:2402.04291}, 2024.

\bibitem{huang2404empirical}
Wei Huang, Xingyu Zheng, Xudong Ma, Haotong Qin, Chengtao Lv, Hong Chen, Jie Luo, Xiaojuan Qi, Xianglong Liu, and Michele Magno.
\newblock An empirical study of llama3 quantization: From llms to mllms, 2024.
\newblock {\em URL https://arxiv. org/abs/2404.14047}.

\bibitem{hubara2016binarized}
Itay Hubara, Matthieu Courbariaux, Daniel Soudry, Ran El-Yaniv, and Yoshua Bengio.
\newblock Binarized neural networks.
\newblock {\em Advances in neural information processing systems}, 29, 2016.

\bibitem{hubara2021accurate}
Itay Hubara, Yury Nahshan, Yair Hanani, Ron Banner, and Daniel Soudry.
\newblock Accurate post training quantization with small calibration sets.
\newblock In {\em International Conference on Machine Learning}, pages 4466--4475. PMLR, 2021.

\bibitem{iverson1962programming}
Kenneth~E Iverson.
\newblock A programming language.
\newblock In {\em Proceedings of the May 1-3, 1962, spring joint computer conference}, pages 345--351, 1962.

\bibitem{karacan2015image}
Levent Karacan, Aykut Erdem, and Erkut Erdem.
\newblock Image matting with kl-divergence based sparse sampling.
\newblock In {\em Proceedings of the IEEE international conference on computer vision}, pages 424--432, 2015.

\bibitem{transfiner}
Lei Ke, Martin Danelljan, Xia Li, Yu-Wing Tai, Chi-Keung Tang, and Fisher Yu.
\newblock Mask transfiner for high-quality instance segmentation.
\newblock In {\em CVPR}, 2022.

\bibitem{vmt}
Lei Ke, Henghui Ding, Martin Danelljan, Yu-Wing Tai, Chi-Keung Tang, and Fisher Yu.
\newblock Video mask transfiner for high-quality video instance segmentation.
\newblock In {\em ECCV}, 2022.

\bibitem{ke2020green}
Zhanghan Ke, Kaican Li, Yurou Zhou, Qiuhua Wu, Xiangyu Mao, Qiong Yan, and Rynson~WH Lau.
\newblock Is a green screen really necessary for real-time portrait matting?
\newblock {\em arXiv preprint arXiv:2011.11961}, 2020.

\bibitem{ke2022modnet}
Zhanghan Ke, Jiayu Sun, Kaican Li, Qiong Yan, and Rynson~WH Lau.
\newblock Modnet: Real-time trimap-free portrait matting via objective decomposition.
\newblock In {\em Proceedings of the AAAI Conference on Artificial Intelligence}, volume~36, pages 1140--1147, 2022.

\bibitem{levin2007closed}
Anat Levin, Dani Lischinski, and Yair Weiss.
\newblock A closed-form solution to natural image matting.
\newblock {\em IEEE transactions on pattern analysis and machine intelligence}, 30(2):228--242, 2007.

\bibitem{li2022vmformer}
Jiachen Li, Vidit Goel, Marianna Ohanyan, Shant Navasardyan, Yunchao Wei, and Humphrey Shi.
\newblock Vmformer: End-to-end video matting with transformer.
\newblock {\em arXiv preprint arXiv:2208.12801}, 2022.

\bibitem{li2023videomatt}
Jiachen Li, Marianna Ohanyan, Vidit Goel, Shant Navasardyan, Yunchao Wei, and Humphrey Shi.
\newblock Videomatt: A simple baseline for accessible real-time video matting.
\newblock In {\em Proceedings of the IEEE/CVF Conference on Computer Vision and Pattern Recognition}, pages 2177--2186, 2023.

\bibitem{lin2022siman}
Mingbao Lin, Rongrong Ji, Zihan Xu, Baochang Zhang, Fei Chao, Chia-Wen Lin, and Ling Shao.
\newblock Siman: Sign-to-magnitude network binarization.
\newblock {\em IEEE Transactions on Pattern Analysis and Machine Intelligence}, 45(5):6277--6288, 2022.

\bibitem{lin2021real}
Shanchuan Lin, Andrey Ryabtsev, Soumyadip Sengupta, Brian~L Curless, Steven~M Seitz, and Ira Kemelmacher-Shlizerman.
\newblock Real-time high-resolution background matting.
\newblock In {\em Proceedings of the IEEE/CVF Conference on Computer Vision and Pattern Recognition}, pages 8762--8771, 2021.

\bibitem{BGMv2_github}
Shanchuan Lin, Andrey Ryabtsev, Soumyadip Sengupta, Brian~L Curless, Steven~M Seitz, and Ira Kemelmacher-Shlizerman.
\newblock Backgroundmattingv2, 05 2023.

\bibitem{lin2022robust}
Shanchuan Lin, Linjie Yang, Imran Saleemi, and Soumyadip Sengupta.
\newblock Robust high-resolution video matting with temporal guidance.
\newblock In {\em WACV}, 2022.

\bibitem{rvm_github}
Shanchuan Lin, Linjie Yang, Imran Saleemi, and Soumyadip Sengupta.
\newblock Robustvideomatting, 05 2023.

\bibitem{lin2015neural}
Zhouhan Lin, Matthieu Courbariaux, Roland Memisevic, and Yoshua Bengio.
\newblock Neural networks with few multiplications.
\newblock {\em arXiv preprint arXiv:1510.03009}, 2015.

\bibitem{liu2020boosting}
Jinlin Liu, Yuan Yao, Wendi Hou, Miaomiao Cui, Xuansong Xie, Changshui Zhang, and Xian-sheng Hua.
\newblock Boosting semantic human matting with coarse annotations.
\newblock In {\em Proceedings of the IEEE/CVF Conference on Computer Vision and Pattern Recognition}, pages 8563--8572, 2020.

\bibitem{liu2022bit}
Zechun Liu, Barlas Oguz, Aasish Pappu, Lin Xiao, Scott Yih, Meng Li, Raghuraman Krishnamoorthi, and Yashar Mehdad.
\newblock Bit: Robustly binarized multi-distilled transformer.
\newblock {\em Advances in neural information processing systems}, 35:14303--14316, 2022.

\bibitem{liu2020reactnet}
Zechun Liu, Zhiqiang Shen, Marios Savvides, and Kwang-Ting Cheng.
\newblock Reactnet: Towards precise binary neural network with generalized activation functions.
\newblock In {\em ECCV}, 2020.

\bibitem{BiReal}
Zechun Liu, Baoyuan Wu, Wenhan Luo, Xin Yang, Wei Liu, and Kwang-Ting Cheng.
\newblock Bi-real net: Enhancing the performance of 1-bit cnns with improved representational capability and advanced training algorithm.
\newblock In {\em ECCV}, 2018.

\bibitem{martinez2020training}
Brais Martinez, Jing Yang, Adrian Bulat, and Georgios Tzimiropoulos.
\newblock Training binary neural networks with real-to-binary convolutions.
\newblock In {\em International Conference on Learning Representations}, 2020.

\bibitem{RealtoBin}
Brais Martinez, Jing Yang, Adrian Bulat, and Georgios Tzimiropoulos.
\newblock Training binary neural networks with real-to-binary convolutions.
\newblock In {\em ICLR}, 2020.

\bibitem{nagel2020up}
Markus Nagel, Rana~Ali Amjad, Mart Van~Baalen, Christos Louizos, and Tijmen Blankevoort.
\newblock Up or down? adaptive rounding for post-training quantization.
\newblock In {\em International Conference on Machine Learning}, pages 7197--7206. PMLR, 2020.

\bibitem{paszke2019pytorch}
Adam Paszke, Sam Gross, Francisco Massa, Adam Lerer, James Bradbury, Gregory Chanan, Trevor Killeen, Zeming Lin, Natalia Gimelshein, Luca Antiga, et~al.
\newblock Pytorch: An imperative style, high-performance deep learning library.
\newblock {\em Advances in neural information processing systems}, 32, 2019.

\bibitem{phan2020binarizing}
Hai Phan, Zechun Liu, Dang Huynh, Marios Savvides, Kwang-Ting Cheng, and Zhiqiang Shen.
\newblock Binarizing mobilenet via evolution-based searching.
\newblock In {\em Proceedings of the IEEE/CVF Conference on Computer Vision and Pattern Recognition}, pages 13420--13429, 2020.

\bibitem{qiao2020attention}
Yu Qiao, Yuhao Liu, Xin Yang, Dongsheng Zhou, Mingliang Xu, Qiang Zhang, and Xiaopeng Wei.
\newblock Attention-guided hierarchical structure aggregation for image matting.
\newblock In {\em Proceedings of the IEEE/CVF Conference on Computer Vision and Pattern Recognition}, pages 13676--13685, 2020.

\bibitem{qin2022bibert}
Haotong Qin, Yifu Ding, Mingyuan Zhang, Qinghua Yan, Aishan Liu, Qingqing Dang, Ziwei Liu, and Xianglong Liu.
\newblock Bibert: Accurate fully binarized bert.
\newblock {\em arXiv preprint arXiv:2203.06390}, 2022.

\bibitem{qin2023diverse}
Haotong Qin, Yifu Ding, Xiangguo Zhang, Jiakai Wang, Xianglong Liu, and Jiwen Lu.
\newblock Diverse sample generation: Pushing the limit of generative data-free quantization.
\newblock {\em IEEE Transactions on Pattern Analysis and Machine Intelligence}, 45(10):11689--11706, 2023.

\bibitem{Qin_2020_pr}
Haotong Qin, Ruihao Gong, Xianglong Liu, Xiao Bai, Jingkuan Song, and Nicu Sebe.
\newblock Binary neural networks: A survey.
\newblock {\em Pattern Recognition}, 2020.

\bibitem{qin2020forward}
Haotong Qin, Ruihao Gong, Xianglong Liu, Mingzhu Shen, Ziran Wei, Fengwei Yu, and Jingkuan Song.
\newblock Forward and backward information retention for accurate binary neural networks.
\newblock In {\em CVPR}, 2020.

\bibitem{qin2024bimatting}
Haotong Qin, Lei Ke, Xudong Ma, Martin Danelljan, Yu-Wing Tai, Chi-Keung Tang, Xianglong Liu, and Fisher Yu.
\newblock Bimatting: Efficient video matting via binarization.
\newblock {\em Advances in Neural Information Processing Systems}, 36, 2024.

\bibitem{qin2024accurate}
Haotong Qin, Xudong Ma, Xingyu Zheng, Xiaoyang Li, Yang Zhang, Shouda Liu, Jie Luo, Xianglong Liu, and Michele Magno.
\newblock Accurate lora-finetuning quantization of llms via information retention.
\newblock {\em arXiv preprint arXiv:2402.05445}, 2024.

\bibitem{qin2023bibench}
Haotong Qin, Mingyuan Zhang, Yifu Ding, Aoyu Li, Ziwei Liu, Fisher Yu, and Xianglong Liu.
\newblock Bibench: Benchmarking and analyzing network binarization.
\newblock In {\em ICML}, 2023.

\bibitem{qin2023distribution}
Haotong Qin, Xiangguo Zhang, Ruihao Gong, Yifu Ding, Yi Xu, and Xianglong Liu.
\newblock Distribution-sensitive information retention for accurate binary neural network.
\newblock {\em International Journal of Computer Vision}, 131(1):26--47, 2023.

\bibitem{XNORNet}
Mohammad Rastegari, Vicente Ordonez, Joseph Redmon, and Ali Farhadi.
\newblock Xnor-net: Imagenet classification using binary convolutional neural networks.
\newblock In {\em ECCV}, 2016.

\bibitem{rosenblatt1961principles}
Frank Rosenblatt.
\newblock Principles of neurodynamics. perceptrons and the theory of brain mechanisms.
\newblock Technical report, Cornell Aeronautical Lab Inc Buffalo NY, 1961.

\bibitem{sandler2018mobilenetv2}
Mark Sandler, Andrew~G. Howard, Menglong Zhu, Andrey Zhmoginov, and Liang{-}Chieh Chen.
\newblock Mobilenetv2: Inverted residuals and linear bottlenecks.
\newblock In {\em CVPR}, 2018.

\bibitem{sengupta2020background}
Soumyadip Sengupta, Vivek Jayaram, Brian Curless, Steven~M Seitz, and Ira Kemelmacher-Shlizerman.
\newblock Background matting: The world is your green screen.
\newblock In {\em Proceedings of the IEEE/CVF Conference on Computer Vision and Pattern Recognition}, pages 2291--2300, 2020.

\bibitem{shang2022lipschitz}
Yuzhang Shang, Dan Xu, Bin Duan, Ziliang Zong, Liqiang Nie, and Yan Yan.
\newblock Lipschitz continuity retained binary neural network.
\newblock In {\em ECCV}, 2022.

\bibitem{shang2022network}
Yuzhang Shang, Dan Xu, Ziliang Zong, Liqiang Nie, and Yan Yan.
\newblock Network binarization via contrastive learning.
\newblock In {\em ECCV}, 2022.

\bibitem{shwartz2022information}
Ravid Shwartz-Ziv.
\newblock Information flow in deep neural networks.
\newblock {\em arXiv preprint arXiv:2202.06749}, 2022.

\bibitem{shwartz2017opening}
Ravid Shwartz-Ziv and Naftali Tishby.
\newblock Opening the black box of deep neural networks via information.
\newblock {\em arXiv preprint arXiv:1703.00810}, 2017.

\bibitem{slonim2002information}
Noam Slonim.
\newblock {\em The information bottleneck: Theory and applications}.
\newblock PhD thesis, Citeseer, 2002.

\bibitem{sun2021deep}
Yanan Sun, Guanzhi Wang, Qiao Gu, Chi-Keung Tang, and Yu-Wing Tai.
\newblock Deep video matting via spatio-temporal alignment and aggregation.
\newblock In {\em Proceedings of the IEEE/CVF Conference on Computer Vision and Pattern Recognition}, pages 6975--6984, 2021.

\bibitem{tailor2020degree}
Shyam~Anil Tailor, Javier Fernandez-Marques, and Nicholas~Donald Lane.
\newblock Degree-quant: Quantization-aware training for graph neural networks.
\newblock In {\em International Conference on Learning Representations}, 2020.

\bibitem{tian2023designing}
Keyu Tian, Yi Jiang, Qishuai Diao, Chen Lin, Liwei Wang, and Zehuan Yuan.
\newblock Designing bert for convolutional networks: Sparse and hierarchical masked modeling.
\newblock {\em arXiv preprint arXiv:2301.03580}, 2023.

\bibitem{wang2023binary}
Junfu Wang, Yuanfang Guo, Liang Yang, and Yunhong Wang.
\newblock Binary graph convolutional network with capacity exploration.
\newblock {\em IEEE Transactions on Pattern Analysis and Machine Intelligence}, 2023.

\bibitem{wang2021gradient}
Qi Wang, Nianhui Guo, Zhitong Xiong, Zeping Yin, and Xuelong Li.
\newblock Gradient matters: Designing binarized neural networks via enhanced information-flow.
\newblock {\em IEEE Transactions on Pattern Analysis and Machine Intelligence}, 44(11):7551--7562, 2021.

\bibitem{wang2024matting}
Zhixiang Wang, Baiang Li, Jian Wang, Yu-Lun Liu, Jinwei Gu, Yung-Yu Chuang, and Shin'Ichi Satoh.
\newblock Matting by generation.
\newblock In {\em ACM SIGGRAPH 2024 Conference Papers}, pages 1--11, 2024.

\bibitem{wang2019learning}
Ziwei Wang, Jiwen Lu, Chenxin Tao, Jie Zhou, and Qi Tian.
\newblock Learning channel-wise interactions for binary convolutional neural networks.
\newblock In {\em Proceedings of the IEEE/CVF conference on computer vision and pattern recognition}, pages 568--577, 2019.

\bibitem{wang2021learning}
Ziwei Wang, Jiwen Lu, Ziyi Wu, and Jie Zhou.
\newblock Learning efficient binarized object detectors with information compression.
\newblock {\em IEEE Transactions on Pattern Analysis and Machine Intelligence}, 44(6):3082--3095, 2021.

\bibitem{wang2020bidet}
Ziwei Wang, Ziyi Wu, Jiwen Lu, and Jie Zhou.
\newblock Bidet: An efficient binarized object detector.
\newblock In {\em CVPR}, 2020.

\bibitem{wei2023deep}
Tianyi Wei, Dongdong Chen, Wenbo Zhou, Jing Liao, Hanqing Zhao, Weiming Zhang, Gang Hua, and Nenghai Yu.
\newblock Deep image matting with sparse user interactions.
\newblock {\em IEEE Transactions on Pattern Analysis and Machine Intelligence}, 2023.

\bibitem{wu2018fast}
Huikai Wu, Shuai Zheng, Junge Zhang, and Kaiqi Huang.
\newblock Fast end-to-end trainable guided filter.
\newblock In {\em Proceedings of the IEEE Conference on Computer Vision and Pattern Recognition}, pages 1838--1847, 2018.

\bibitem{wu2022seqformer}
Junfeng Wu, Yi Jiang, Song Bai, Wenqing Zhang, and Xiang Bai.
\newblock Seqformer: Sequential transformer for video instance segmentation.
\newblock In {\em European Conference on Computer Vision}, pages 553--569. Springer, 2022.

\bibitem{xu2017deep}
Ning Xu, Brian Price, Scott Cohen, and Thomas Huang.
\newblock Deep image matting.
\newblock In {\em Proceedings of the IEEE conference on computer vision and pattern recognition}, pages 2970--2979, 2017.

\bibitem{xu2021poem}
Sheng Xu, Yanjing Li, Junhe Zhao, Baochang Zhang, and Guodong Guo.
\newblock Poem: 1-bit point-wise operations based on expectation-maximization for efficient point cloud processing.
\newblock {\em arXiv preprint arXiv:2111.13386}, 2021.

\bibitem{xu2021recu}
Zihan Xu, Mingbao Lin, Jianzhuang Liu, Jie Chen, Ling Shao, Yue Gao, Yonghong Tian, and Rongrong Ji.
\newblock Recu: Reviving the dead weights in binary neural networks.
\newblock In {\em Proceedings of the IEEE/CVF international conference on computer vision}, pages 5198--5208, 2021.

\bibitem{yang2025matanyone}
Peiqing Yang, Shangchen Zhou, Jixin Zhao, Qingyi Tao, and Chen~Change Loy.
\newblock {MatAnyone}: Stable video matting with consistent memory propagation.
\newblock In {\em CVPR}, 2025.

\bibitem{yin2019understanding}
Penghang Yin, Jiancheng Lyu, Shuai Zhang, Stanley Osher, Yingyong Qi, and Jack Xin.
\newblock Understanding straight-through estimator in training activation quantized neural nets.
\newblock In {\em International Conference on Learning Representations}, 2019.

\bibitem{yu2021mask}
Qihang Yu, Jianming Zhang, He Zhang, Yilin Wang, Zhe Lin, Ning Xu, Yutong Bai, and Alan Yuille.
\newblock Mask guided matting via progressive refinement network.
\newblock In {\em Proceedings of the IEEE/CVF Conference on Computer Vision and Pattern Recognition}, pages 1154--1163, 2021.

\bibitem{zhang2021diversifying}
Xiangguo Zhang, Haotong Qin, Yifu Ding, Ruihao Gong, Qinghua Yan, Renshuai Tao, Yuhang Li, Fengwei Yu, and Xianglong Liu.
\newblock Diversifying sample generation for accurate data-free quantization.
\newblock In {\em Proceedings of the IEEE/CVF conference on computer vision and pattern recognition}, pages 15658--15667, 2021.

\bibitem{zhang2021attention}
Yunke Zhang, Chi Wang, Miaomiao Cui, Peiran Ren, Xuansong Xie, Xian-Sheng Hua, Hujun Bao, Qixing Huang, and Weiwei Xu.
\newblock Attention-guided temporally coherent video object matting.
\newblock In {\em Proceedings of the 29th ACM International Conference on Multimedia}, pages 5128--5137, 2021.

\bibitem{zhao2019improving}
Ritchie Zhao, Yuwei Hu, Jordan Dotzel, Chris De~Sa, and Zhiru Zhang.
\newblock Improving neural network quantization without retraining using outlier channel splitting.
\newblock In {\em International conference on machine learning}, pages 7543--7552. PMLR, 2019.

\bibitem{dorefa}
Shuchang Zhou, Yuxin Wu, Zekun Ni, Xinyu Zhou, He Wen, and Yuheng Zou.
\newblock Dorefa-net: Training low bitwidth convolutional neural networks with low bitwidth gradients.
\newblock {\em CoRR}, abs/1606.06160, 2016.

\bibitem{zhuang2018towards}
Bohan Zhuang, Chunhua Shen, Mingkui Tan, Lingqiao Liu, and Ian~D Reid.
\newblock Towards effective low-bitwidth convolutional neural networks.
\newblock {\em CVPR}, 2018.

\bibitem{zou2019unsupervised}
Dongqing Zou, Xiaowu Chen, Guangying Cao, and Xiaogang Wang.
\newblock Unsupervised video matting via sparse and low-rank representation.
\newblock {\em IEEE transactions on pattern analysis and machine intelligence}, 42(6):1501--1514, 2019.

\end{thebibliography}
\bibliographystyle{ieee_fullname}
}

\clearpage


%

\begin{IEEEbiography}[{\includegraphics[width=1in,height=1.25in,clip,keepaspectratio]{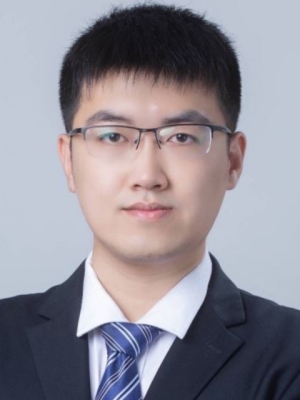}}]{Haotong Qin} is a Postdoctoral Researcher at the Center for Project-Based Learning (PBL) D-ITET at ETH Zürich. He received his B.S. and Ph.D. degrees at Beihang University. His research interests include model compression and deployment toward efficient deep learning in real-world scenarios. He has published over 50 papers in top-tier journals and conferences, such as IEEE TPAMI, IEEE TNNLS, IJCV, ICML, NeurIPS, ICLR, and CVPR. He serves as Guest Editor for Neural Networks, \textit{etc}., Area Chair in NeurIPS, ACM MM, AISTATS, \textit{etc}., and Reviewer for IEEE TPAMI, IEEE TIP, IEEE TNNLS, IJCV, NeurIPS, ICML, ICLR, \textit{etc}.
\end{IEEEbiography}

\begin{IEEEbiography}[{\includegraphics[width=1in,height=1.25in,clip,keepaspectratio]{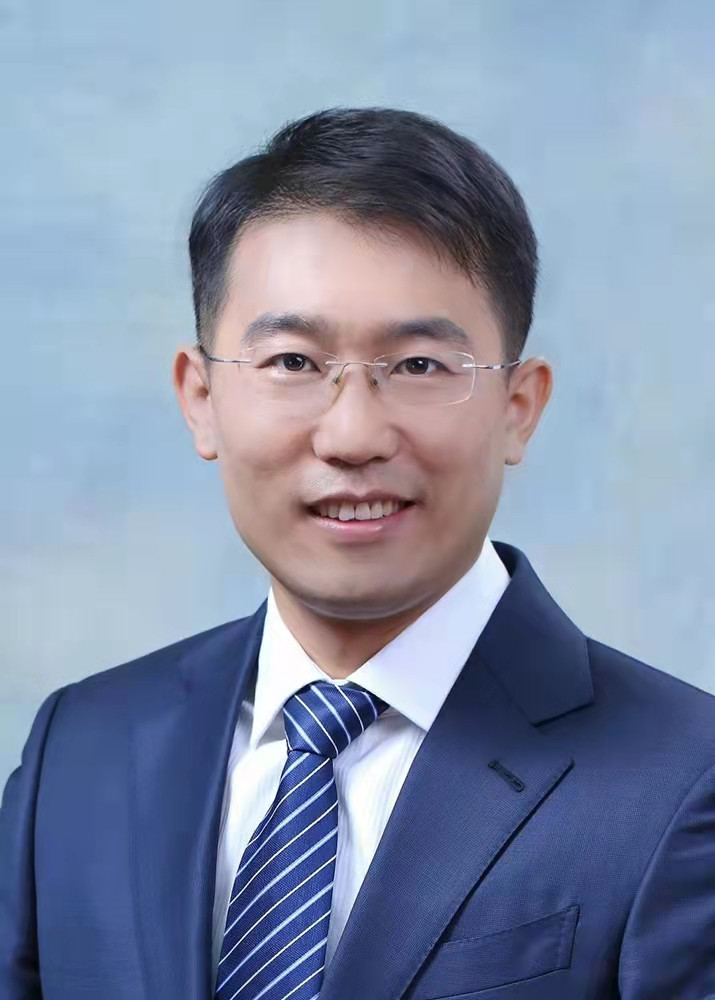}}]{Xianglong Liu} (Senior Member, IEEE) is currently a Professor, serves as the vice dean of the School of Computer Science and Engineering at Beihang University, and is also the deputy director of the State Key Laboratory of Complex and Critical Software Environment. He is the recipient of the China National Excellent Youth Science Fund. He has published over 100 papers in top conferences/journals in artificial intelligence and information security, such as NeurIPS, ICLR, CVPR, ICCV, CSS, and IJCV. He serves as Associate Editor and Guest for several SCI journals like Pattern Recognition and IET Image Processing and as Promotion Editor for journals like Frontiers of Computer Science and Acta Aeronautica et Astronautica Sinica. He serves as Area Chair in top conferences such as AAAI and ACM MM and has frequently organized workshops and competitions in conferences like CVPR, IJCAI, and AAAI.
\end{IEEEbiography}

\begin{IEEEbiography}[{\includegraphics[width=1in,height=1.25in,clip,keepaspectratio]{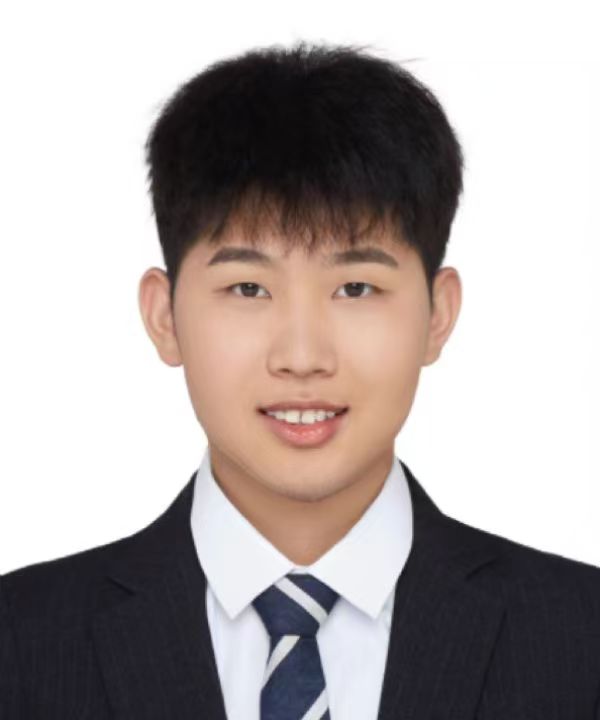}}]{Xudong Ma} received the B.E. degree in Network Engineering from the University of Electronic Science and Technology of China. He is currently working toward a PhD degree in the State Key Laboratory of Complex and Critical Software Environment at Beihang University under the supervision of Prof. Jie Luo. His research interests include model compression, knowledge graphs, and language models. He has published papers in NeurIPS, ICML, IJCAI, and IEEE TNNLS.
\end{IEEEbiography}

\begin{IEEEbiography}[{\includegraphics[width=1in,height=1.25in,clip,keepaspectratio]{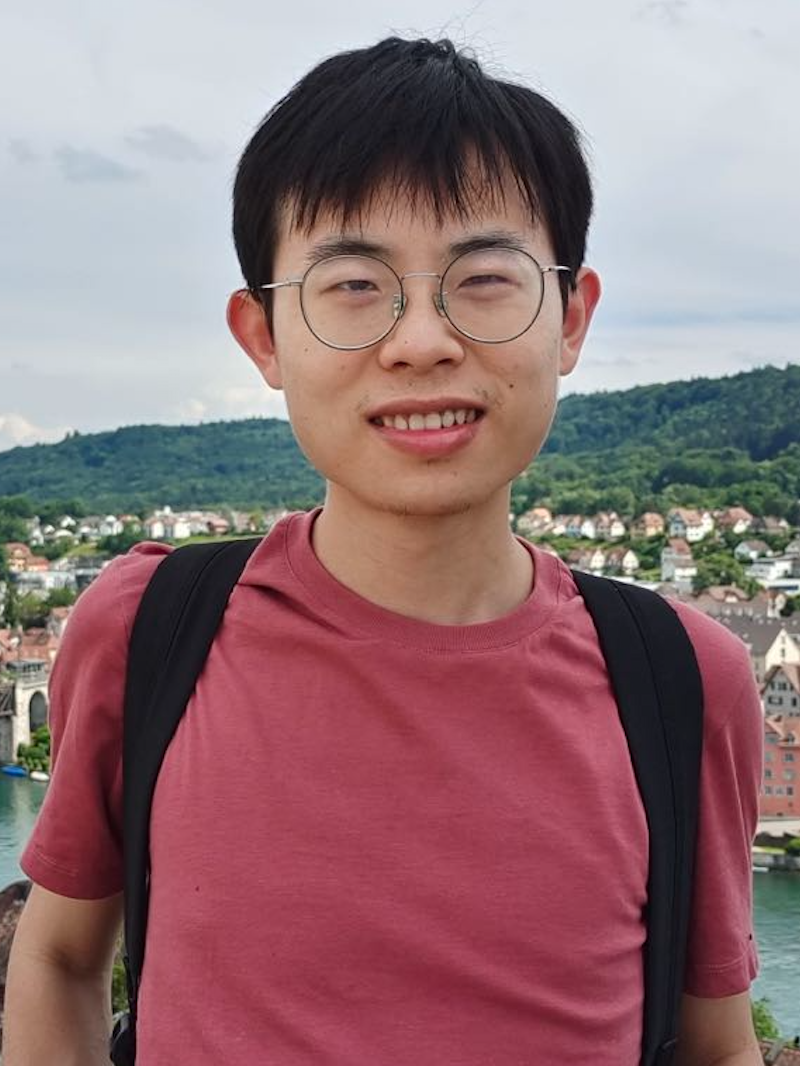}}]{Lei Ke} is a Postdoctoral Research Associate at the Machine Learning Department of Carnegie Mellon University.
Previously, he was a postdoc researcher at Computer Vision Lab of ETH Zurich in 2023.
He received his Ph.D. degree in the Department of Computer Science and Engineering at the Hong Kong University of Science and Technology. He also visited the Computer Vision Laboratory of ETH Zürich for 2 years during his PhD study. His research goal is to enable machines to achieve 4D and multi-modality scene understanding from videos/images and further perform robot actions. 
His leading research projects have accumulated over 7K stars on GitHub.
\end{IEEEbiography}

\begin{IEEEbiography}[{\includegraphics[width=1in,height=1.25in,clip,keepaspectratio]{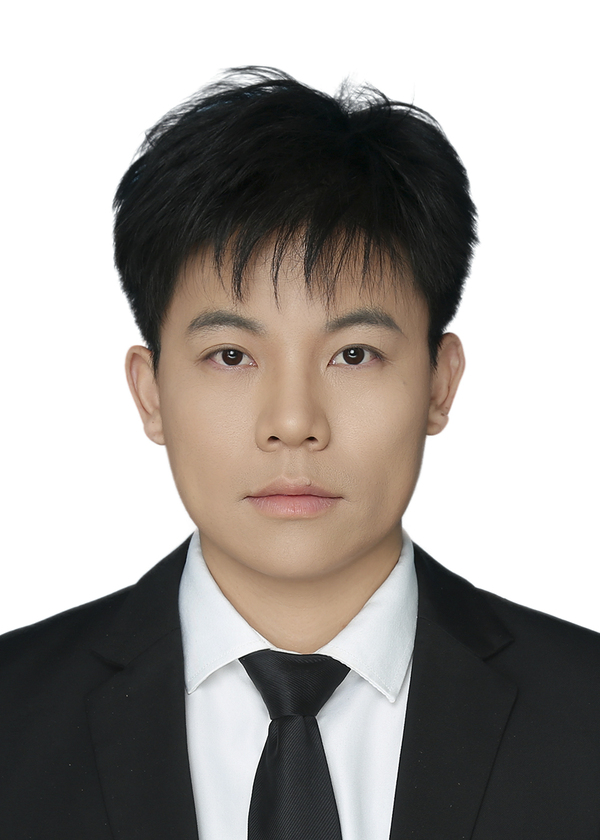}}]{Yulun Zhang} received a B.E. degree from the School of Electronic Engineering, Xidian University, China, in 2013, an M.E. degree from the Department of Automation, Tsinghua University, China, in 2017, and a Ph.D. degree from the Department of ECE, Northeastern University, USA, in 2021. He is an associate professor at Shanghai Jiao Tong University, Shanghai, China. He was a postdoctoral researcher at Computer Vision Lab, ETH Zürich, Switzerland. His research interests include image/video restoration and synthesis, biomedical image analysis, model compression, multimodal computing, large language model, and computational imaging. He is/was an Area Chair for CVPR, ICCV, ECCV, NeurIPS, ICML, ICLR, IJCAI, ACM MM, and a Senior Program Committee (SPC) member for IJCAI and AAAI.
\end{IEEEbiography}

\begin{IEEEbiography}[{\includegraphics[width=1in,height=1.25in,clip,keepaspectratio]{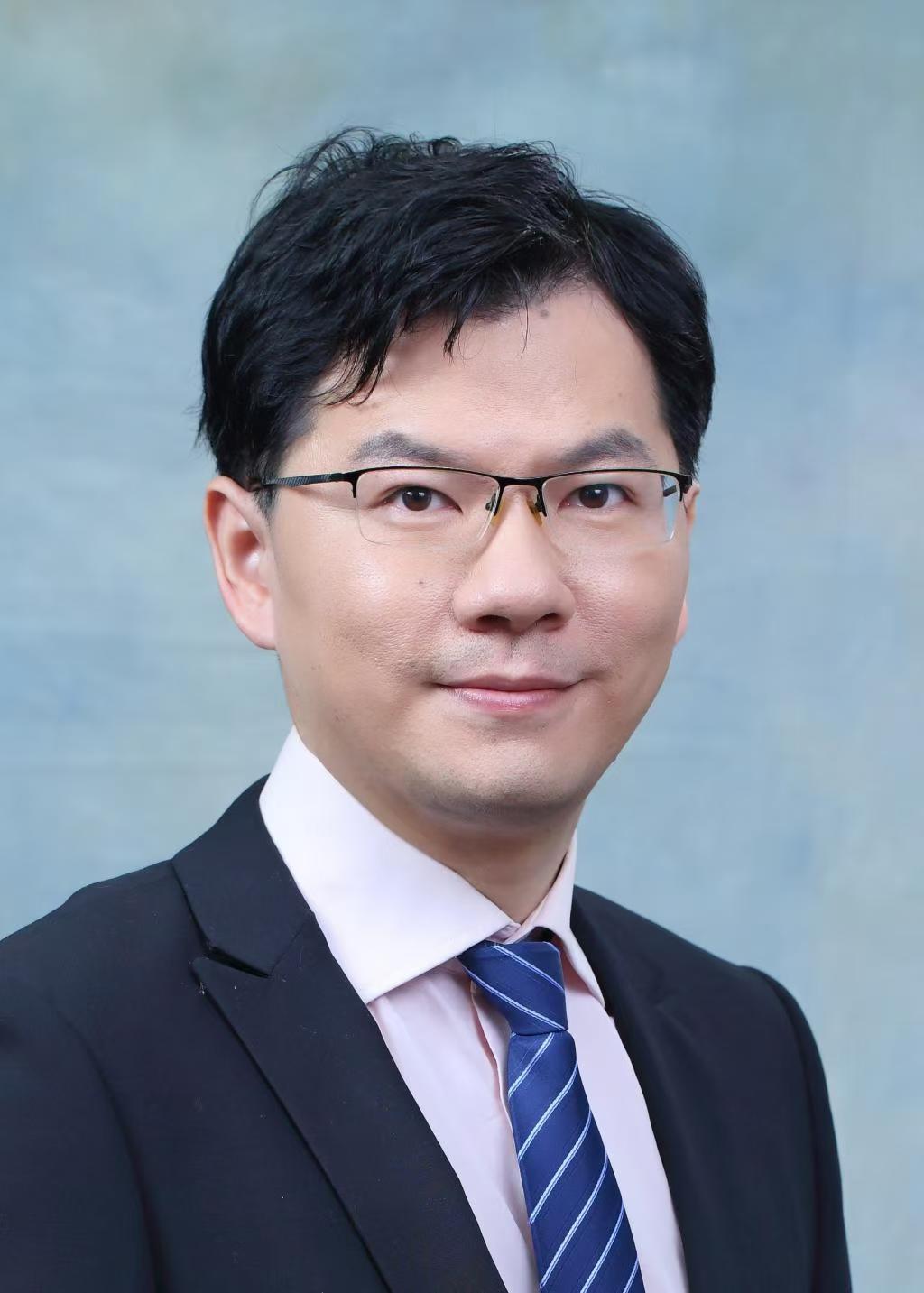}}]{Jie Luo} is now an Associate Professor in State Key Laboratory of Complex and Critical Software Environment, School of Computer Science and Engineering, Beihang University. He received a B.S. degree from the School of Mathematical Sciences, Peking University in 2003. And he received a Ph.D. degree from Beihang University in 2012, under the supervision of Prof. Wei Li and visited the University of Washington as a joint PhD student. His research interests include logic foundation for computer science, knowledge engineering, and crowd intelligence. He has published over 40 papers in top conferences and journals in artificial intelligence and information security, such as NeurIPS, ICML, CVPR, IJCAI, BiBM, KSEM, and PR.
\end{IEEEbiography}

\begin{IEEEbiography}[{\includegraphics[width=1in,height=1.25in,clip,keepaspectratio]{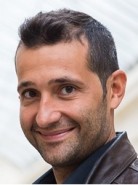}}]{Michele Magno} (Senior Member, IEEE) is currently a Privatdozent with the Department of Information Technology and Electrical Engineering (D-ITET), ETH Zürich, where he has been leading the D-ITET Center for project-based learning since 2020. He is also a senior member of IEEE. He received master's and Ph.D. degrees in electronic engineering from the University of Bologna, Italy, in 2004 and 2010, respectively. Since 2013, he has been with ETH Zürich, Switzerland, and has become a Visiting Lecturer or a Professor at the University of Nice Sophia, France; Enssat Lannion, France; the University of Bologna, Italy; and Mid University Sweden, Sweden; where is a Full Visiting Professor with the Department of Electrical Engineering. He has authored more than 300 papers in international journals and conferences. Some of his publications were awarded Best Paper awards at several IEEE conferences.
\end{IEEEbiography}

\clearpage

    \appendices

\section{Overview}
We provide additional information in this appendix. Section~\ref{sec:proof_of_thm} presents detailed proofs of the theorems discussed in the main text. Section~\ref{sec:arch} comprehensively explains our neural architecture of BiVM. Finally, in Section~\ref{sec:vis}, we include further results for in-depth comparison and visual examples of our composited matting data samples.

\section{Main Proofs}
\label{sec:proof_of_thm}
\subsection{Proof of Theorem 1}
\label{app:thm1}
\begin{thm}
Let $X$ be a normally distributed random variable. 
Consider the two functions
$\hat{T} = f(\operatorname{sign}(X))$ and $T = f(X)$,
where $\operatorname{sign}(\cdot)$ denotes the sign function and $f(x) = ax + b\ (a \neq 0)$, then we have $I(X; \hat{T}) \ll I(X; T)$.
\end{thm}

\begin{proof}
First, consider $I(X; T)$:
$T = aX + b$ is a linear transformation of $X$ with $a \neq 0$. Therefore:
$$I(X; T) = H(X) - H(X|T) = H(X)$$
where $H(\cdot)$ denotes information entropy.
For $X\sim\mathcal{N}(\mu, \sigma^2)$:
\begin{equation}
\label{eq:th1_p0}
\begin{aligned}
H(X) = \lim\limits_{\Delta\rightarrow0}H(X_\Delta)= -\int p(x) \log p(x) dx - \lim\limits_{\Delta\rightarrow0}\log\Delta,
\end{aligned}
\end{equation}
where $p(x)= \frac{1}{\sqrt{2\pi\sigma^2}} \exp\left( -\frac{(x - \mu)^2}{2\sigma^2} \right)$, $X_\Delta$ is a discretized version of $X$ with bin width $\Delta$, and $\Delta$ is the discretization interval (or bin width) for approximating the entropy of a continuous random variable by a discrete one.
Following Eq.~\eqref{eq:th1_p0}, we thus have 
\begin{equation}
\label{eq:th1_p0.1}
H(X)=\infty.
\end{equation}

Now, consider $I(X; \hat{T})$:
$$\hat{T} = (aX + b) \cdot \text{sign}(X),$$
We can decompose $X$ into its sign and absolute value:
$$X = |X| \cdot \text{sign}(X).$$
Using the data processing inequality and the chain rule of mutual information:
\begin{equation}
\label{eq:th1_p1}
\begin{aligned}
    I(X; \hat{T}) \leqslant &I(|X|, \text{sign}(X); \hat{T})\\
    = &I(|X|; \hat{T}) + I(\text{sign}(X); \hat{T} | |X|).
\end{aligned}
\end{equation}
Note that $\text{sign}(X)$ is a function of $X$, so:
\begin{equation}
\label{eq:th1_p2}
I(\text{sign}(X); \hat{T} | |X|) \leqslant H(\text{sign}(X)) \leqslant \log(2).
\end{equation}
For $I(|X|; \hat{T})$, observe that given $|X|$, $\hat{T}$ provides at most 1 bit of additional information (the sign of $X$). Thus:
\begin{equation}
\label{eq:th1_p3}
I(|X|; \hat{T}) \leqslant H(|X|) - [H(|X|) - \log(2)] = \log(2).
\end{equation}
Combining steps Eq.~\eqref{eq:th1_p1}, Eq.~\eqref{eq:th1_p2}, and Eq.~\eqref{eq:th1_p3}:
\begin{equation}
\label{eq:th1_p3.1}
I(X; \hat{T}) \leqslant 2\log(2).
\end{equation}

Therefore, we have $I(X; \hat{T}) \ll I(X; T)$ considering Eq.~\eqref{eq:th1_p0.1} and Eq.~\eqref{eq:th1_p3.1}. Note that the $X$ is usually represented as 32-bit in practice. In this case, since $p(x) \log p(x) \leqslant 0$ and $\log\frac{1}{2^{32}}\ll 0$, $I(X; \hat{T}) \ll I(X; T)$ still hold.
\end{proof}

\subsection{Proof of Theorem 2}

\begin{thm}
Let $X$ be an input variable and $\{f_k : k \in \mathbb{Z} \cap [1, N]\}$ be a set of functions. Define the representation $T_n$ as:
$T_n = \prod_{k=1}^n f_k(X)$
Then, for any $i, j \in \mathbb{Z} \cap [1, N]$ with $i < j$, we have:
$I(X; T_i) \geqslant I(X; T_j)$.
\end{thm}

\begin{proof}
First, note that for $i < j$, we can express $T_j$ as
$$T_j = \prod_{k=i+1}^j f_k(T_i) = \prod_{k=i+1}^j f_k\left( \prod_{q=1}^i f_q(X)\right).$$
Then we can form a Markov chain $X \rightarrow T_i \rightarrow T_j$.
This is because $T_j$ is a function of $T_i$ and $Y$, where $Y$ is a function of $X$. Given $T_i$, $T_j$ depends on $X$ only through $Y$.

By the chain rule of mutual information, we can write
\begin{equation}
\label{eq:th2_p0}
I(X;T_i,T_j) = I(X;T_i) + I(X;T_j|T_i),
\end{equation}
And we can also write
\begin{equation}
\label{eq:th2_p1}
I(X;T_i,T_j) = I(X;T_j) + I(X;T_i|T_j).
\end{equation}
Since $X \rightarrow T_i \rightarrow T_j$ forms a Markov chain, we have 
\begin{equation}
\label{eq:th2_p2}
I(X;T_j|T_i) = 0
\end{equation}
(given $T_i$, $T_j$ provides no additional information about $X$).
Thus, we have
\begin{equation}
\label{eq:th2_p3}
I(X;T_i,T_j) = I(X;T_i).
\end{equation}
Since mutual information is always non-negative, we have 
\begin{equation}
\label{eq:th2_p4}
I(X;T_i|T_j) \geqslant 0.
\end{equation}
Considering Eq.~\eqref{eq:th2_p1}, Eq.~\eqref{eq:th2_p3}, and Eq.~\eqref{eq:th2_p4}:
$$I(X;T_i) = I(X;T_i,T_j) \geqslant I(X;T_j)$$
Thus, we get
$$I(X;T_i) \geqslant I(X;T_j).$$

Therefore, we have proved that for any $i, j \in \mathbb{Z} \cap [0, N]$ with $i < j$, $I(X; T_i) \geqslant I(X; T_j)$.
\end{proof}

\section{Details of Network Architecture}
\label{sec:arch}

\begin{table*}[!ht]
    \centering
    \resizebox{1\columnwidth}{!}{
    \begin{tabular}{l|c|ccc|ccc|ccc}
    \toprule
        Module & Conv\tiny{32bit}  & ~ & EBB\_1 $\psi$ & ~ & ~ & EBB\_2 $\psi$ & ~   \\ \midrule
        Sub-Module & ~ & Sub-EBB $\theta^\uparrow$ & Sub-EBB $\theta^\uparrow$ & Sub-EBB $\theta^\downarrow$ & Sub-EBB $\theta^\uparrow$ & Sub-EBB $\theta^\uparrow$ & Sub-EBB $\theta^\downarrow$ \\ \midrule
        In/Out Channel & (3, 16) & (16, 32) & (32, 64) & (64, 32) & (32, 64) & (64, 128) & (128, 64) \\ \midrule
        Extracted Feature & $\frac{1}{2}$ & ~ & ~ & $\frac{1}{4}$ & ~ & ~ & ~  \\ \midrule\midrule
        Module & ~ & ~ & EBB\_3 $\psi$ & ~ & ~ & EBB\_4 $\psi$ & ~  \\ \midrule
        Sub-Module & ~ & Sub-EBB $\theta^-$ & Sub-EBB $\theta^\uparrow$ & Sub-EBB $\theta^\downarrow$ & Sub-EBB $\theta^\uparrow$ & Sub-EBB $\theta^\uparrow$ & Sub-EBB $\theta^\downarrow$ \\ \midrule
        In/Out Channel & ~ & (64, 64) & (64, 128) & (128, 64) & (64, 128) & (128, 256) & (256, 128) \\ \midrule
        Extracted Feature & ~ & ~ & ~ & $\frac{1}{8}$ & ~ & ~ \\ \midrule\midrule
        Module & ~ & ~ & EBB\_5 $\psi$ & ~ & ~ & ~ & ASPP  \\ \midrule
        Sub-Module & ~ & Sub-EBB $\theta^-$ & Sub-EBB $\theta^\uparrow$ & Sub-EBB $\theta^\downarrow$ & Sub-EBB $\theta^\uparrow$ & Sub-EBB $\theta^\uparrow$ &  \\ \midrule
        In/Out Channel & ~ & (128, 128) & (128, 256) & (256, 128) & (128, 256) & (256, 1024) & (1024, 128) \\ \midrule
        Extracted Feature & ~ & ~ & ~ & ~ & ~ & $\frac{1}{16}$ \\ \bottomrule
    \end{tabular}}
    \vspace{0.05in}
    \caption{The encoder structure in our BiVM is detailed here, where "Extracted Feature" refers to utilizing features derived by this sub-block within the encoder. Sub-EBB and EBB are labeled according to the notations provided in Eq.~\eqref{eq:sub-ebb} and Eq.~\eqref{eq:ebb} of our paper, representing the sub-block and the whole block, respectively.
    }
    \label{tab:encoder}
\end{table*}

\begin{table*}[!ht]
    \centering
    \resizebox{1\columnwidth}{!}{
    \begin{tabular}{l|c|c|c|c|c}
    \toprule
        Module & BottleNeck & SHB\_1 (Upsampling) & SHB\_2 (Upsampling) & SHB\_3 (Upsampling) & SHB\_4 (Output)  \\ \midrule
        Feature Scale & $\frac{1}{16}$ & $\frac{1}{8}$ & $\frac{1}{4}$ & $\frac{1}{2}$ & $\frac{1}{1}$  \\ \midrule
        Input Mask & ~ & $\boldsymbol{m}_\text{inc}^{\tau^*}$ & $\boldsymbol{m}_\text{inc}^{\tau^*}$ & $\boldsymbol{m}_\text{inc}^{\tau^*}$ & $\boldsymbol{m}_\text{inc}^{\tau^*}$  \\ \midrule
        Produced Mask & $\boldsymbol{m}_\text{inc}^{\tau^*}$  & ~ & ~ & ~ & ~ \\ \bottomrule
    \end{tabular}}
    \vspace{0.05in}
    \caption{The decoder structure in our BiVM is detailed here. SHBs are labeled according to our paper's notations in Eq.~\eqref{eq:SHB}. The sparse mask $\boldsymbol{m}_\text{inc}^{\tau^*}$ is employed to guide the decoder's computations, particularly focusing on the ``difficult" regions.
    }
    \label{tab:decoder}
\end{table*}

\begin{figure*}[t]
  \centering
  \includegraphics[width=\textwidth]{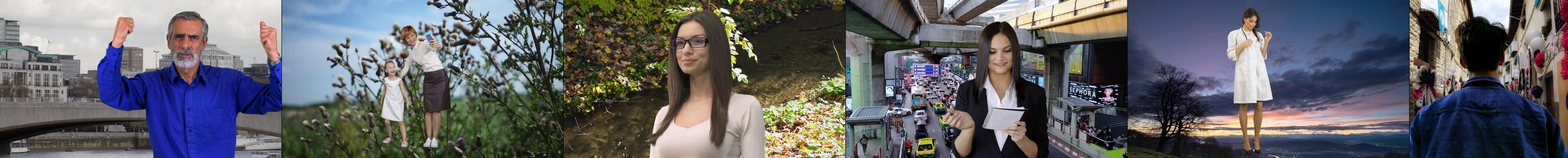}
  \caption{ Example testing samples.}
  \label{fig:test-sample}
\end{figure*}

As Table~\ref{tab:encoder} shows, the binarized encoder constructed by our EBB processes individual frames to generate feature maps across multiple spatial scales, ranging from $\frac{1}{2}$ down to $\frac{1}{16}$. Within each EBB, the feature channel count either remains unchanged or doubles in the first sub-block double again in the second sub-block and is then halved in the third sub-block. This configuration allows each EBB to extract features in a higher-dimensional channel space than the input. The dense computational structure ensures that binarized convolutions effectively capture high-quality features.
In contrast to the full-precision MobileNetV3 backbones that operate at a $\frac{1}{32}$ scale, we altered the final block by employing convolutions with a dilation rate of 2 and a stride of 1, following the principles outlined in~\cite{lin2022robust}. The resulting feature map at the $\frac{1}{16}$ scale is then passed to the LR-ASPP module, which reduces it to 128 channels.

For the decoder in Table~\ref{tab:decoder}, as described in our paper, the SHB is utilized in every block except the first to enhance computational efficiency. The binary mask used by the SHB is derived from the initial non-sparse binarized block, which operates at the smallest feature scale, thus generating the mask with minimal computational overhead. This design choice significantly boosts the decoder's efficiency.

In addition, the Deep Guided Filter (DGF) incorporates a limited number of binarized $1\times 1$ convolutions. For detailed specifications, please refer to~\cite{wu2018fast,lin2022robust}. The entire network is implemented and trained using PyTorch~\cite{paszke2019pytorch}.

\section{Additional Visualizations}
\label{sec:vis}

\begin{figure*}[t]
  \centering
  \includegraphics[width=\textwidth]{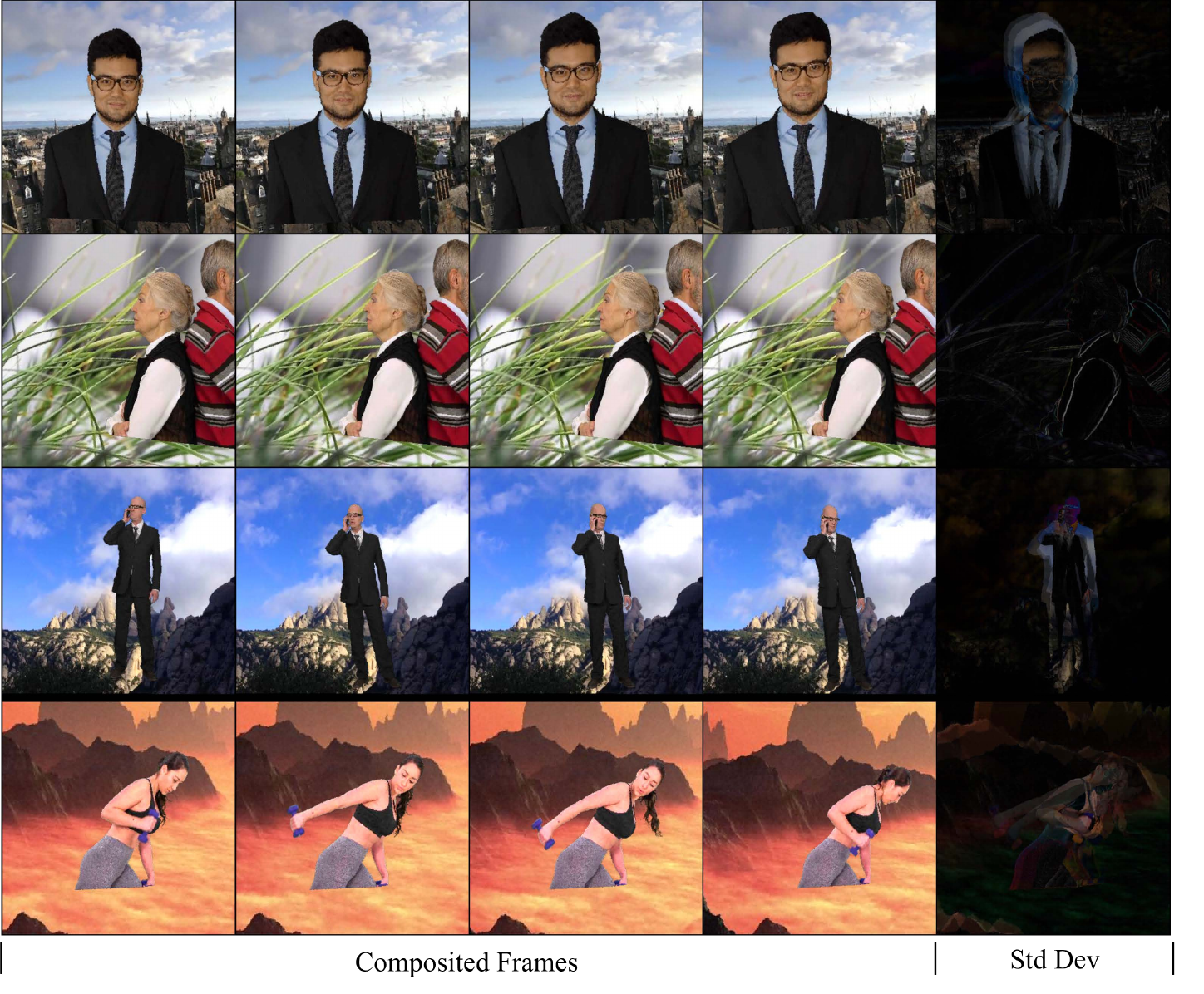}
  \caption{Composite training samples. The last column is the pixels' temporal standard deviation.}
  \label{fig:train-sample}
\end{figure*}

\begin{figure*}[t]
  \centering
  \includegraphics[width=\textwidth]{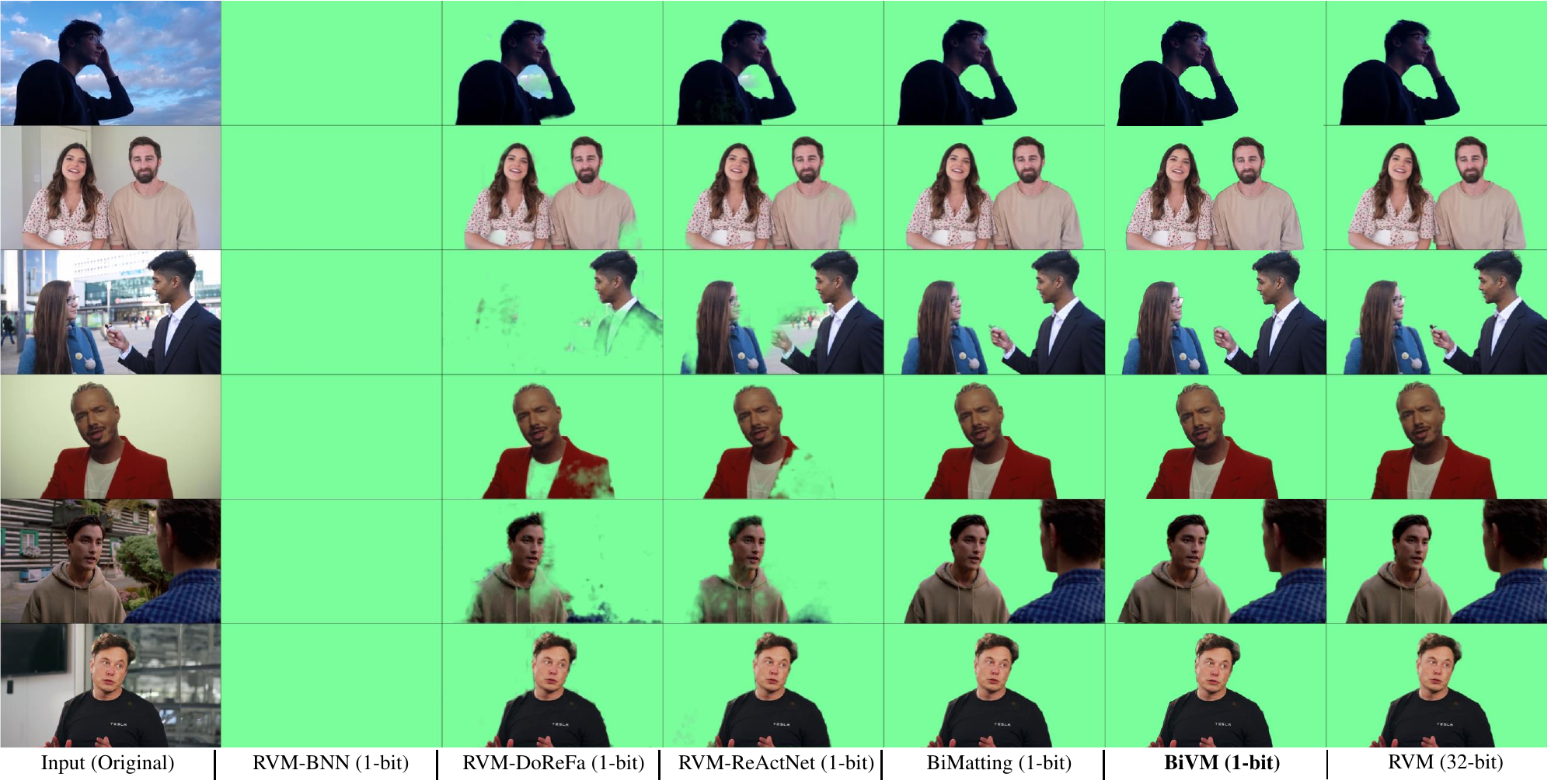}
  \caption{More visual results. Compared to 1-bit video matting networks using existing binarization methods, our BiVM surpasses them and achieves near full-precision performance. Note that the results of RVM-BNN indicate the model fully crashes.}
  \label{fig:comparsion}
\end{figure*}

\subsection{Composited Datasets}
\label{sec:data}

As guided by \cite{lin2022robust}, we constructed composite training and testing samples. Examples of composited training samples from the matting datasets are displayed in Figure~\ref{fig:train-sample}. These clips exhibit natural movements when composited with videos and artificial movements generated through motion augmentation.
For the testing samples in Figure~\ref{fig:test-sample}, motion augmentation was applied exclusively to the foreground and background elements of the images. This augmentation was limited to affine transformations. Additionally, the intensity of the augmentation was intentionally reduced compared to the training phase, ensuring that the testing samples maintained a high level of realism.

\subsection{Visual Results}
\label{subsec:vis_results}

Additional visual results are presented in Figure~\ref{fig:comparsion}, where the superiority of our BiVM over other binarization methods becomes even more apparent. Specifically, our method excels in preserving intricate edge details and achieving more accurate matting in local regions, which are critical for high-quality image processing tasks. The comparisons highlight how BiVM maintains sharpness and clarity in challenging areas, such as fine edges and complex textures, where other methods may falter. These improvements underscore the effectiveness of our approach in producing visually superior results across a range of scenarios.

{We also present a supplementary \href{https://drive.google.com/file/d/1Jk7wrX2re9DYG51eDa7wf-9H0r0ETMzu/view?usp=drive_link}{\textcolor{blue}{video}} to assess the temporal consistency for our proposed BiVM thoroughly. The supplemental video we provide contains four different scenarios, comparing the input (original), our BiVM (1-bit), RVM (32-bit), RVM-BNN (1-bit), RVM-ReActNet (1-bit), and BiMatting (1-bit, the previous best model, also the conference version method). We find that our BiVM exhibits outstanding temporal consistency and frame-wise matting performance. BiVM not only consistently outperforms existing 1-bit binarized models, including the previously best BiMatting model, but its visual quality even approaches the full-precision (32-bit) RVM. For example, at the 9th second of the video (Figure~\ref{fig:video}), none of the binarized models handle the region enclosed by the body and arm very well; even the BiMatting begins to produce imprecise segmentation in the late part of that second. By contrast, our 1-bit BiVM visually demonstrates excellent temporal coherence and segmentation accuracy similar to the original 32-bit RVM. In summary, the visual results in the video further highlight the advantages of our proposed BiVM.}

\begin{figure*}[!h]
  \centering
  \includegraphics[width=1\textwidth]{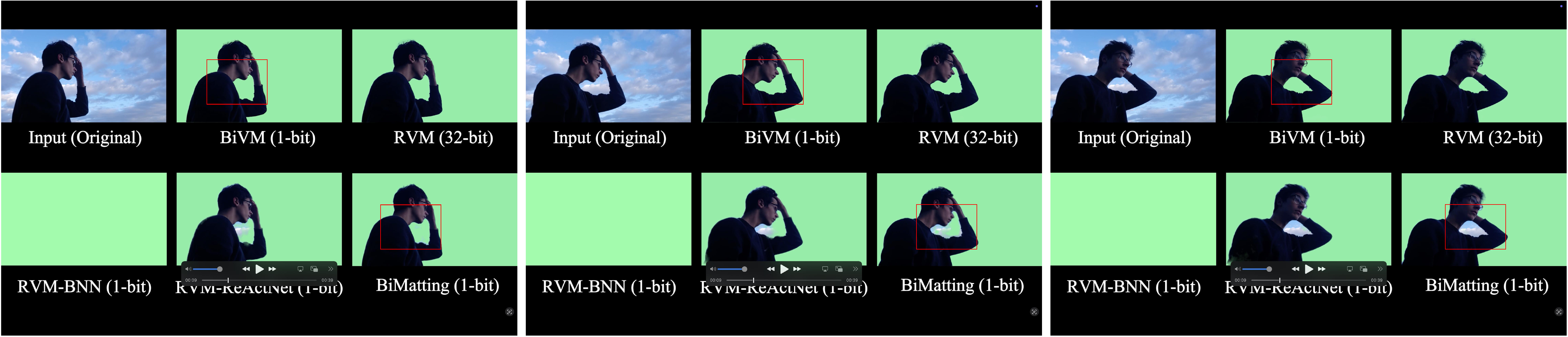}
  \caption{{Frames in the 9th second of the supplemental video.}}
  \label{fig:video}
\end{figure*}

\end{document}